\documentclass{article}

    \usepackage[preprint,nonatbib]{neurips_2024}

\usepackage[utf8]{inputenc} 
\usepackage[T1]{fontenc}    
\usepackage{hyperref}       
\usepackage{url}            
\usepackage{booktabs}       
\usepackage{amsfonts}       
\usepackage{nicefrac}       
\usepackage{microtype}      
\usepackage{xcolor}         
\usepackage{subcaption}

\title{Practical $0.385$-Approximation for Submodular Maximization Subject to a Cardinality Constraint}

\author{Murad Tukan
  \\
  DataHeroes Israel\\
  \texttt{murad@dataheroes.ai} \\
  \And 
  Loay Mualem \\
  Department of Computer Science\\
  University of Haifa\\
  Haifa Israel\\
  \texttt{loaymual@gmail.com}\\  
  \And Moran Feldman\\
  Department of Computer Science\\
  University of Haifa\\
  Haifa Israel\\
  \texttt{moranfe@cs.haifa.ac.il}\\
}

\usepackage{amsmath}
\usepackage{mathtools}
\usepackage[linesnumbered,algoruled,boxed,lined,noend]{algorithm2e}
\usepackage{amsthm}
\usepackage{thmtools} 
\usepackage{amssymb}
\usepackage{thm-restate}

\newcommand{\say}[1]{``#1''}
\newcommand{\term}[1]{\left( #1 \right)}
\newcommand{\f}[1]{f\term{#1}}

\newcommand{\eps}{\varepsilon}
\newcommand{\OPT}{\mathbb{O}\mathbb{P}\mathbb{T}}
\newcommand{\br}[1]{\left\lbrace #1 \right\rbrace}
\newcommand{\abs}[1]{\left| #1 \right|}
\DeclarePairedDelimiter\ceil{\lceil}{\rceil}
\DeclarePairedDelimiter\floor{\lfloor}{\rfloor}
\newtheorem{theorem}{Theorem}[section]

\newtheorem{lemma}[theorem]{Lemma}

\DeclareMathOperator*{\argmax}{arg\,max}
\DeclareMathOperator*{\argmin}{arg\,min}
\newcommand{\REAL}{\mathbb{R}}
\newcommand{\T}[1]{T_\lambda\term{#1}}
\newcommand{\sterm}[1]{\left[ #1 \right]}
\newcommand{\E}[1]{\mathbb{E}\sterm{#1}}
\newcommand{\lovasz}[1]{\hat{f}\term{#1}}
\newcommand{\x}[1]{\mathbf{x}_{#1}}
\newcommand{\norm}[1]{\left\| #1 \right\|}
\newcommand{\bigO}[1]{\mathcal{O}\term{#1}}
\newcommand{\localsearch}{\textsc{Local-Search}}
\newcommand{\flocalsearch}{\textsc{Fast-Local-Search}}
\newcommand{\Stochasticgreedy}{\textsc{Aided-Measured Discrete Stochastic Greedy}}

\newcommand{\pr}[1]{\mathrm{Pr}\term{#1}}

\SetKwBlock{With}{With}{end}

\makeatletter
\newcommand{\defcal}[1]{\expandafter\newcommand\csname c#1\endcsname{{\mathcal{#1}}}}
\newcommand{\defbb}[1]{\expandafter\newcommand\csname b#1\endcsname{{\mathbb{#1}}}}
\newcommand{\defvec}[1]{\expandafter\newcommand\csname v#1\endcsname{{\mathbf{#1}}}}
\newcommand{\defmat}[1]{\expandafter\newcommand\csname m#1\endcsname{{\mathbf{#1}}}}
\makeatother

\usepackage[inline]{enumitem}

\begin{document}

\maketitle

\begin{abstract}
Non-monotone constrained submodular maximization plays a crucial role in various machine learning applications. However, existing algorithms often struggle with a trade-off between approximation guarantees and practical efficiency. The current state-of-the-art is a recent $0.401$-approximation algorithm, but its computational complexity makes it highly impractical. The best practical algorithms for the problem only guarantee $1/e$-approximation. In this work, we present a novel algorithm for submodular maximization subject to a cardinality constraint that combines a guarantee of $0.385$-approximation with a low and practical query complexity of $O(n+k^2)$. Furthermore, we evaluate the empirical performance of our algorithm in experiments based on various machine learning applications, including Movie Recommendation, Image Summarization, and more. These experiments demonstrate the efficacy of our approach.
\end{abstract}

\section{Introduction}

In the last few years, the ability to effectively summarize data has gained importance due to the advent of massive datasets in many fields. Such summarization often consists of selecting a small representative subset from a large corpus of images, text, movies, etc. Without a specific structure, this task can be as challenging as finding a global minimum of a non-convex function. Fortunately, many practical machine learning problems exhibit some structure, making them suitable for optimization techniques (either exact or approximate).

A key structure present in many such problems is submodularity, also known as the principle of diminishing returns. This principle suggests that the incremental value of an element decreases as the set it is added to grows. Submodularity enables the creation of algorithms that can provide near-optimal solutions, making it fundamental in machine learning. It has been successfully applied to various tasks, such as social graph analysis~\cite{norouzi2018beyond}, adversarial attacks~\cite{lei2019discrete,mualem2024submodular}, dictionary learning~\cite{das2011submodular}, data summarization~\cite{mitrovic2018data,mualem2023resolving,mualem2024bridging}, interpreting neural networks~\cite{elenberg2017streaming}, robotics~\cite{zhou2022risk,tukan2023orbslam3}, and many more. 

To exemplify the notion of submodularity, consider the following task. Given a large dataset, our goal is to identify a subset that effectively summarizes (or covers) the data, with a good representative set being one that covers the majority of the data. Note that adding an element $s$ to a set $A$ is less beneficial to this goal than adding it to a subset $B \subset A$ due to the higher likelihood of overlapping coverage. Formally, if $\mathcal{N}$ is the set of elements in the dataset, and we define a function $f\colon 2^{\mathcal{N}} \rightarrow \mathbb{R}$ mapping every set of elements to its coverage, then, the above discussion implies that, for every two sets $A \subseteq B \subseteq \mathcal{N}$ and element $s \in \mathcal{N} \setminus B$, it must hold that $f(s \mid A) \geq f(s \mid B)$, where $f(s \mid A) \triangleq f(\{s\} \cup A) - f(A)$ denotes the marginal gain of the element $s$ to the set $A$. We say that a set function is \emph{submodular} if it obeys this property.

Unfortunately, maximizing submodular functions is NP-hard even without a constraint~\cite{feige2011maximizing}, and therefore, works on maximization of such functions aim for  approximations. Many of these works make the extra assumption that the submodular function $f\colon 2^{\mathcal{N}} \to \mathbb{R}$ is \emph{monotone}, i.e., that for every two sets $A\subseteq B\subseteq \mathcal{N}$, it holds that $f(B)\geq f(A)$. Two of the first works of this kind, by Nemhauser and Wolsey~\cite{nemhauser1978best} and Nemhauser et al.~\cite{nemhauser1978analysis}, showed that a greedy algorithm achieves a tight $1-\nicefrac{1}{e}$ approximation for the problem of maximizing a non-negative monotone submodular function subject to a cardinality constraint using $O(nk)$ function evaluations, where $n$ is the size of the ground set $\mathcal{N}$ and $k$ is the maximum cardinality allowed for the output set. An important line of work aimed to improve over the time complexity of the last algorithm, culminating with various deterministic and randomized algorithms that have managed to reduce the time complexity to linear at the cost of an approximation guarantee that is worse only by a $1 - \eps$ factor~\cite{buchbinder2017comparing,mirzasoleiman2015lazier,li2022submodular,kuhnle2021quick,huang2022multi}. 

Unfortunately, the submodular functions that arise in machine learning applications are often non-monotone, either because they are naturally non-monotone, or because a diversity-promoting non-monotone regularizer is added to them. 
Maximizing a non-monotone submodular function is challenging. The only tight approximation known for such functions is for the case of unconstrained maximization, which enjoys a tight approximation ratio of $\nicefrac{1}{2}$~\cite{feige2011maximizing}. A slightly more involved case is the problem of maximizing a non-negative (not necessarily monotone) submodular function subject to a cardinality constraint. This problem has been studied extensively, first, Lee et al.~\cite{lee2009nonmonotone} suggested an algorithm guaranteeing $(1/4 - \eps)$-approximation for this problem. This approximation ratio was improved in a long series of works~\cite{buchbinder2014submodular,chekuri2014submodular,ene2016constrained,vondrak2013symmetry}, leading to the very recent $0.401$-approximation algorithm of Buchbinder and Feldman~\cite{buchbinder2023constrained}, which improved over a previous $0.385$-approximation algorithm due to Buchbinder and Feldman~\cite{buchbinder2019constrained}. On the inapproximability side, it has been shown that no algorithm can guarantee a better approximation ratio than $0.478$ in polynomial time~\cite{qi2022maximizing}.

Most of the results in the above-mentioned line of work are only of theoretical interest due to a very high time complexity. The two exceptions are the Random Greedy algorithm of Buchbinder et al.~\cite{buchbinder2014submodular} that guarantees $\nicefrac{1}{e}$-approximation using $O(nk)$ queries to the objective function, and the Sample Greedy algorithm of Buchbinder et al.~\cite{buchbinder2017comparing} that reduces the query complexity to $O_\eps(n)$ at the cost of a slightly worse approximation ratio of $\nicefrac{1}{e} - \eps$.

\subsection{Our contribution}
In this work, we introduce a novel combinatorial algorithm for maximizing a non-negative submodular function subject to a cardinality constraint. Our suggested method combines a practical query complexity of $O(n+k^2)$ with an approximation guarantee of $0.385$, which improves over the $\nicefrac{1}{e}$-approximation of the state-of-the-art practical algorithm. To emphasize the effectiveness of our suggested method, we evaluate it on $5$ applications: (i) Movie Recommendation, (ii) Image Summarization, and (iii) Revenue Maximization.
Our experiments demonstrate that Algorithm~\ref{alg:faster_local_search} outperforms the current practical state-of-the-art algorithms.

\textbf{Remark.} An independent work that recently appeared on arXiv~\cite{chen2024guided} suggests another $0.385$-approximation algorithm for our problem using $O(nk)$ oracle queries. Interestingly, their algorithm is very similar to a basic version of our algorithm presented in the appendix. In this work, our main goal was to find ways to speed up this basic algorithm, which leads to our main result. In contrast, the main goal of~\cite{chen2024guided} was to derandomize the basic algorithm and extend it to other constraints.

\section{Preliminaries}

In this section, we define some additional notation used throughout the paper. Given an element $u \in \mathcal{N}$ and a set $S \subseteq \mathcal{N}$, we use $S + u $ and $S - u$ as shorthands for $S \cup \{u\}$ and $S \setminus \{u\}$. 
Given a set function $f\colon 2^\mathcal{N} \to \mathbb{R}$, we use $f(u \mid S)$ to denote the marginal contribution of $u$ to $S$. 
Similarly, given an additional set $T \subseteq \mathcal{N}$, we define $f(T \mid S) \triangleq f(S \cup T) - f(S)$.

\section{Method}\label{sec:main_result}
In this section, we present our algorithm for non-monotone submodular maximization under cardinality constraints, which is the algorithm used to prove the main theoretical result of our work (Theorem~\ref{thm:main}). We begin with a brief overview of our algorithm.

Motivated by the ideas underlying the impractical $0.385$-approximation algorithm of~\cite{buchbinder2019constrained}, our algorithm comprises three steps:

\begin{enumerate}
    \item \textbf{Initial Solution:} We start by searching for a good initial solution that guarantees a constant approximation to the optimal set. This is accomplished by running the Sample Greedy algorithm of~\cite{buchbinder2017comparing} $O(\log \eps^{-1})$ times, and selecting the solution with the highest function value.

    \item \textbf{Accelerated Local Search (Algorithm~\ref{alg:faster_local_search}):} Next, the algorithm aims to find an (approximate) local optimum set $Z$ using a local search method. This can be done using a classical local search algorithm, at the cost of $O(nk^2)$ queries (see Appendix~\ref{alg:local_search} for more detail). As an alternative, we introduce, in Subsection~\ref{subsection:fastlocal}, our accelerated local search algorithm {\flocalsearch} (Algorithm~\ref{alg:faster_local_search}), which reduces the query complexity to $O(n+k^2)$.

    \item \textbf{Accelerated Stochastic Greedy Improvement (Algorithm~\ref{alg:main_lazier_than_lazy_greedy}):} After obtaining the set $Z$, the algorithm attempts to improve the solution using a stochastic greedy algorithm. By using for this purpose a version of the Random Greedy algorithm suggested by Buchbinder et al.~\cite{buchbinder2014submodular}, one could get our target approximation ratio of $0.385$ using $O(nk)$ queries (we refer the reader to Appendix~\ref{subsection:fastlocal} for the complete proof). To get the same result using fewer queries, we employ Algorithm~\ref{alg:main_lazier_than_lazy_greedy} (described in Subsection~\ref{subsection:guided}), which is accelerated using ideas borrowed from the Sample Greedy Algorithm of~\cite{buchbinder2017comparing}.
    
\end{enumerate}

Our final algorithm (given as Algorithm~\ref{alg:385_fast} in Subsections~\ref{ssc:main}) returns the better among the two sets produced in the last two steps (i.e., the output sets of Algorithm~\ref{alg:faster_local_search}, and Algorithm~\ref{alg:main_lazier_than_lazy_greedy}).

\subsection{Fast local search}\label{subsection:fastlocal}
In this section, we present our accelerated local search algorithm, which is the algorithm used to implement the first two steps of our main algorithm. The properties of this algorithm are formally given by Theorem~\ref{thm:fast_local_search}. Let $\OPT$ be an optimal solution.

\begin{restatable}{theorem}{thmfastlocalsearch} 
\label{thm:fast_local_search}
There exists an algorithm that given a positive integer $k$, a value $\eps \in (0, 1)$, and a non-negative submodular function $f\colon 2^{\mathcal{N}} \to \mathbb{R}$, outputs a set $S \subseteq \mathcal{N}$ of size at most $k$ that, with probability at least $1-\eps$, obeys
\[\f{S} \geq \frac{\f{S \cap \OPT} + \f{S \cup \OPT}}{2 + \eps} \quad \text{and} \quad  \f{S} \geq \frac{\f{S \cap \OPT}}{1 + \eps}\enspace.\]
Furthermore, the query complexity of the above algorithm is $O_{\eps}(n + k^2)$. 
\end{restatable}

Note that the guarantee of Theorem~\ref{thm:fast_local_search} is similar to the guarantee of a classical local search algorithm (see Appendix~\ref{sec:warmup_methods} for details). However, such a classical local search algorithm uses $O_\eps(nk^2)$ queries, which is much higher than the number of queries required for the algorithm from Theorem~\ref{thm:fast_local_search}.

We defer the formal proof of Theorem~\ref{thm:fast_local_search} to Appendix~\ref{sec:proofs_main}. The proof of Theorem~\ref{thm:fast_local_search} is based on Algorithm~\ref{alg:faster_local_search}. Algorithm~\ref{alg:faster_local_search} implicitly assumes that the ground set $\mathcal{N}$ includes at least $2k$ dummy elements that always have a zero marginal contribution to $f$. Such elements can always be added to the ground set (before executing the algorithm) without affecting the properties of $f$, and removing them from the output set of the algorithm does not affect the guarantee of Theorem~\ref{thm:fast_local_search}.

\begin{algorithm}[!t]
    \DontPrintSemicolon
    \SetKwInOut{Input}{input}
    \SetKwInOut{Output}{output}

    \Input{A positive integer $k \geq 1$, a submodular function $f$, an approximation factor $\eps \in (0,1)$, and a number $L$ of iterations.}
    \Output{A subset of $\mathcal{N}$ of cardinality at most $k$.}

    Initialize $S_0$ to be a feasible solution that with probability at least $1 - \eps$ provides $c$-approximation for the problem for some constant $c \in (0, 1]$.\label{algLine:faster_local_search_init_S}\\
    Fill $S_0$ with dummy elements to ensure $\abs{S_0} = k$.\\
    \For{$j=1$ to $\ceil{\log_2{\frac{1}{\eps}}}$\label{line:attempts_loop}}{
        Let $S_0^j \gets S_0$.\\
        \For{$i = 1$ to $L$}{
            $Z^j_{i} \gets $ Sample $\frac{n}{k}$ items from $\mathcal{N}$ uniformly at random.\\
            $u^j_i \gets \argmax_{u^\prime \in Z^j_{i}} \f{S^j_{i-1} + u^j_i} - \f{S^j_{i-1}}$.\\
            \If{$\f{S^j_{i-1} + u^j_i} - \f{S^j_{i-1}} \leq 0$}{
                $u^j_i \gets $ dummy element that does not belong to $S^j_{i - 1}$.
            }
            $v^j_i \gets \argmin_{v^\prime \in S^j_{i-1}} \f{S^j_{i-1}} - \f{S^j_{i-1} - v^\prime}$. \label{algLine:faster_local_search_deletion}\\
            \lIf{$\f{S^j_{i-1}} < \f{S^j_{i-1} - v^j_i + u^j_i}$}{
                $S^j_{i} \gets S^j_{i-1} - v^j_i + u^j_i$.
            }
            \lElse{$S^j_{i} \gets S^j_{i-1}$.}
        }
        Pick a uniformly random integer $0 \leq i^* < L$.\\
        \If{for every integer $0 \leq t \leq k$ it holds that
        $\max\limits_{S \subseteq \mathcal{N} \setminus S^j_{i^*}, \abs{S} = t} \sum\nolimits_{u \in S} \f{u \mid S^j_{i^*}} \leq \min\limits_{S \subseteq S^j_{i^*}, \abs{S} = t} \sum\nolimits_{v \in S} \f{v \mid S^j_{i^*} - v} + \eps\f{S^j_{i^*}}$\label{line:condition}}{
            \Return $S^j_{i^*}$.
        }
    }
    \Return FAILURE.
    \caption{\flocalsearch$\term{k,f, \eps, L}$}
    \label{alg:faster_local_search}
\end{algorithm}

Algorithm~\ref{alg:faster_local_search} starts by finding an initial solution $S_0$ guaranteeing constant approximation. As mentioned above, this initial solution is found using repeated applications of the Sample Greedy algorithm of~\cite{buchbinder2017comparing}. If the size of the initial solution is less than $k$ (i.e., $|S_0| < k$), the algorithm adds to it $k - |S_0|$ dummy elements. Then, Algorithm~\ref{alg:faster_local_search} makes roughly $\log_2 \eps^{-1}$ attempts to find a good output. Each attempt  improves the initial solution using $L$ iterations consisting of three steps: In Step~(i), the algorithm samples $\frac{n}{k}$ items, and  picks the element $u$ from the sample with the largest marginal contribution to the current solution $S_{i-1}$. If there are no elements with a positive marginal contribution, the algorithm picks a dummy element outside $S_{i-1}$ as $u$. In Step~(ii), the algorithm picks the element $v \in S_{i-1}$ that has the lowest marginal value, i.e., the element whose removal from $S_{i-1}$ would lead to the smallest drop in value. In Step~(iii), the algorithm swaps the elements $u$ and $v$ if such a swap increases the value of the current solution. Once $L$ iterations are over, the algorithm picks a uniformly random solution among all the solutions seen during this attempt (recall that the algorithm makes roughly $\log_2  \eps^{-1}$ attempts to find a good solution). If the random solution found obeys the technical condition given on Line~\ref{line:condition}, then the algorithm returns it. Otherwise, the algorithm continues to the next attempt. If none of the attempts returns a set, the algorithm admits failure.

\subsection{Guided stochastic greedy}\label{subsection:guided}
In this section, we prove Theorem~\ref{thm:lazy_greedy_guarantee}, which provides the last step of our main algorithm.
\begin{restatable}{theorem}{lazierGreedy}
\label{thm:lazy_greedy_guarantee}
There exists an algorithm that given a positive integer $k$, a value $\eps \in (0, 1)$, a value $t_s \in [0, 1]$, a non-negative submodular function $f\colon 2^{\mathcal{N}} \to \mathbb{R}$, and a set $Z \subseteq \mathcal{N}$ obeying the inequalities given in Theorem~\ref{thm:fast_local_search}, outputs a solution $S_k$, obeying 
\begin{align*}
\E{\f{S_k}} &\geq \term{\frac{k - \ceil{t_s \cdot k}}{k} \alpha^{k - \ceil{t_s \cdot k} - 1} + \alpha^{k - \ceil{t_s \cdot k}} - \alpha^{k}}\f{\OPT} + \\
&\quad + \term{\alpha^k + \alpha^{k-1} - \frac{2k - \ceil{t_s \cdot k} - 1}{k} \alpha^{k - \floor{t_s \cdot k}}}\f{\OPT \cup Z} \\
&\quad + \alpha^k \f{S_0} + \term{\alpha^k - \alpha^{k - \ceil{t_s \cdot k}} } \f{\OPT \cap Z} \\
&\quad - 2\eps\term{1 - \alpha^k}\f{\OPT}.    
\end{align*}
Furthermore, this algorithm requires only $O_\eps(n)$ queries to the objective function.
\end{restatable}

\begin{algorithm}[!t]
    \SetKwInOut{Input}{input}
    \SetKwInOut{Output}{output}

    \Input{A set $Z \subseteq \mathcal{N}$, a positive integer $k \geq 1$, values $\eps \in (0, 1)$ and $t_s \in [0,1]$, and a non-negative submodular function $f$}
    \Output{A set $S_k \subseteq \mathcal{N}$}

    Initialize $S_0 \gets \emptyset$. \\
    Let $p \gets 8k^{-1}\eps^{-2}\ln{\term{2\eps^{-1}}}$ \label{alg:lazier_greedy_setting_p}.\\
    Let $s_1 \gets k / \term{n - \abs{Z}}$. \\
    Let $s_2 \gets k / n$. \\
    \For{$i = 1$ to $\ceil{k \cdot t_s}$}{
        Let $M_i \subseteq \mathcal{N} \setminus Z$ be a uniformly random set containing $\ceil{p \cdot \term{n - \abs{Z}}}$ elements.\\
        Let $d_i$ be uniformly random scalar from the range $(0, s_1\ceil{p \term{n - \abs{Z}}}]$.\\
        Let $u_i$ be an element of $M_i$ associated with the $\ceil{d_i}$-the largest marginal contribution to $S_{i-1}$.\\
        \If{$\f{u_i \mid S_{i-1}} \geq 0$}{
            $S_i \gets S_{i-1} \cup \br{u_{i}}$.
        }
        \Else{$S_{i} \gets S_{i-1}$.}
    }
   \For{$i = \ceil{k \cdot t_s} + 1$ to $k$}{
        Let $M_i \subseteq \mathcal{N}$ be a uniformly random set containing $\ceil{p \cdot n}$ elements.\\
        Let $d_i$ be uniformly random scalar from the range $(0, s_2\ceil{ p\cdot n}]$.\\
        Let $u_i$ be an element of $M_i$ associated with the $\ceil{d_i}$-the largest marginal contribution to $S_{i-1}$.\\
        \lIf{$\f{u_i \mid S_{i-1}} \geq 0$}{
            $S_i \gets S_{i-1} \cup \br{u_{i}}$.
        }
        \lElse{$S_{i} \gets S_{i-1}$.}
    }
    
    \caption{Guided Stochastic Greedy}
    \label{alg:main_lazier_than_lazy_greedy}
\end{algorithm}

The algorithm used to prove Theorem~\ref{thm:lazy_greedy_guarantee} is Algorithm~\ref{alg:main_lazier_than_lazy_greedy}. This algorithm starts with an empty set and adds elements to it in iterations (at most one element per iteration) until its final solution is ready after $k$ iterations. In its first $\lceil k \cdot t_s \rceil $ iterations, the algorithm ignores the elements of $Z$, and in the other iterations it considers all elements. However, except for this difference, the behavior of the algorithm in all iterations is very similar. Specifically, in each iteration $i$ the algorithm does the following two steps. In Step~(i), the algorithm samples a subset $M_i$ containing $O_\eps(n/k)$ elements from the data. In Step~(ii), the algorithm considers a subset of $M_i$ (of size either $s_1\lceil p(n - |Z|) \rceil$ or $s_2\lceil pn) \rceil$) containing the elements of $M_i$ with the largest marginal contributions with respect to the current solution $S_{i - 1}$, and adds a uniformly random element out of this subset to the solution (if this element has a positive marginal contribution).

\begin{algorithm}[!t]
    \SetKwInOut{Input}{input}
    \SetKwInOut{Output}{output}
    \DontPrintSemicolon
    \Input{A positive integer $k \geq 1$, a submodular function $f$, error rate $\eps \in (0,1)$, and a flip point $0\leq t_s\leq 1$}
    \Output{A set $S_L \subseteq \mathcal{N}$}
    
    $Z\gets \flocalsearch\term{k,f, \eps, L:= \lceil 16k / (\eps\term{1 - \nicefrac{1}{e}}} \rceil$. \label{algLine:main_run_fast_local_search}\\
    \If{the last algorithm did not fail}{
    $A\gets \Stochasticgreedy\term{Z,k,t_s,\eps}$ \label{algLine:main_run_stochastic_greedy}.\\
    
    \Return $\max\{f(Z),f(A)\}$.}
    \lElse{\Return{$\emptyset$}.}
    \caption{A $0.385$-approximation algorithm for submodular maximiziation}
    \label{alg:385_fast}
\end{algorithm}

\subsection{$0.385$-Approximation guarantee} \label{ssc:main}

In this section, our objective is to prove the following theorem.

\begin{theorem}[Approximation guarantee] \label{thm:main}
Given an integer $k \geq 1$ and a non-negative submodular function $f\colon 2^{\mathcal{N}} \to \mathbb{R}$, there exists an $0.385$-approximation algorithm for the problem of finding a set $S \subseteq \mathcal{N}$ of size at most $k$ maximizing $f$. This algorithm uses $O_\eps\term{n + k^2}$ queries to the objective function.
\end{theorem}

The algorithm used to prove Theorem~\ref{thm:main} is Algorithm~\ref{alg:385_fast}. Specifically, Lemma~\ref{lem:almost_main} gives our guarantee for Algorithm~\ref{alg:385_fast}. Notice that this lemma immediately implies Theorem~\ref{thm:main} when $k$ is large enough. If $k$ is too small, we can still get Theorem~\ref{thm:main} from Lemma~\ref{lem:almost_main} using a three steps process. First, we choose an integer $\rho$ such that $\rho k$ is large enough, and we create a new ground set $\mathcal{N}_\rho = \{u_i \mid u \in \mathcal{N}, i \in [\rho]\}$ and a new objective function $g\colon 2^{\mathcal{N}_\rho} \to \mathbb{R}$ defined as 
$
    g(S)
    =
    \E{f(R(S))}
    ,
$%
where $R(S)$ is a random subset of $\mathcal{N}$ that includes every element $u \in \mathcal{N}$ with probability $|S \cap (\{u\} \times [\rho])| / \rho$. Then, we use Lemma~\ref{lem:almost_main} to get a set $\hat{S}$ that provides $0.385$-approximation for the problem $\max\{g(S) \mid |S| \leq \rho k\}$. Finally, the Pipage Rounding technique of~\cite{calinescu2011calinescu} can be used to get from $\hat{S}$ a $0.385$-approximation for our original problem. Notice that since the size of $\hat{S}$ is constant (as we consider the case of a small $k$), this rounding can be done using a constant number of queries to the objective.

The proof of Theorem ~\ref{thm:main} is based on Lemma~\ref{lem:almost_main}, which we present and prove next.

\begin{lemma}
\label{lem:almost_main}
Algorithm~\ref{alg:385_fast} makes $O_\eps(n + k^2)$ queries to the objective function, and returns a set whose expected value is $c - O(\eps) - O(k^{-1})$ for a constant value $c > 0.385$.
\end{lemma}

\begin{proof}
According to the proof of Theorem~\ref{thm:fast_local_search}, our choice of the parameter $L$ in Algorithm~\ref{alg:faster_local_search} guarantees that with probability at least $1 - 2\eps$ the set $Z$ obeys the inequalities
\[
    f(Z) \geq \frac{f(Z \cup \OPT) + f(Z \cap OPT)}{2 + \eps}
    \qquad
    \text{and}
    \qquad
    f(Z) \geq \frac{f(Z \cap \OPT)}{1 + \eps}
    \enspace.
\]
Let us denote by $\mathcal{E}$ the event that these inequalities hold. By Theorem~\ref{thm:lazy_greedy_guarantee},
\begin{align*}
\E{\f{A} \mid \mathcal{E}} &\geq \term{\frac{k - \ceil{t_s \cdot k}}{k} \alpha^{k - \ceil{t_s \cdot k} - 1} + \alpha^{k - \ceil{t_s \cdot k}} - \alpha^{k}}\f{\OPT} + \\
&\quad + \term{\alpha^k + \alpha^{k-1} - \frac{2k - \ceil{t_s \cdot k} - 1}{k} \alpha^{k - \floor{t_s \cdot k}}}\E{\f{\OPT \cup Z} \mid \mathcal{E}} \\
&\quad + \alpha^k \f{S_0} + \term{\alpha^k - \alpha^{k - \ceil{t_s \cdot k}} } \E{\f{\OPT \cap Z}\mid \mathcal{E}} \\
&\quad - 2\eps\term{1 - \alpha^k}\f{\OPT}\enspace.    
\end{align*}

Since the output of Algorithm~\ref{alg:main_lazier_than_lazy_greedy} is the better set among $A$ and $Z$, we can lower bound its value by any convex combination of lower bounds on the values of $A$ and $Z$. More formally, if we denote by $p_1$, $p_2$ and $p_3$ any three non-negative values that add up to $1$, then we get
\begin{equation} \label{eq:max_inequality}
\begin{split}
\E{\max\{\f{A}, \f{Z}\}\mid \mathcal{E}} \mspace{-150mu}&\mspace{150mu}\geq p_3\term{\frac{k - \ceil{t_s \cdot k}}{k} \alpha^{k - \ceil{t_s \cdot k} - 1} + \alpha^{k - \ceil{t_s \cdot k}} - \alpha^{k}}\f{\OPT} \\
&\quad + \term{\frac{p_1}{2 + \eps} + p_3\term{\alpha^k + \alpha^{k-1} - \frac{2k - \ceil{t_s \cdot k} - 1}{k} \alpha^{k - \floor{t_s \cdot k}}}} \E{\f{\OPT \cup Z} \mid \mathcal{E}} \\
&\quad + \term{\frac{p_2}{1 + \eps} + \frac{p_1}{2 + \eps} - p_3 \term{\alpha^{k - \ceil{t_s \cdot k}} - \alpha^k}} \E{\f{\OPT \cap Z}\mid \mathcal{E}} \\
&\quad - 2\eps p_3 \term{1 - \alpha^k}\f{\OPT}\enspace.
\end{split}
\end{equation}

To simplify the above inequality, we need to lower bound some of the terms in it. First,
\begin{align*}
\frac{k - \ceil{t_s \cdot k}}{k} &\alpha^{k - \ceil{t_s \cdot k} - 1} + \alpha^{k - \ceil{t_s \cdot k}} - \alpha^{k} \geq \term{2 - t_s - \frac{1}{k}} \alpha^{k\term{1 - t_s}} - \alpha^k\\
&\geq  \term{2 - t_s - \frac{1}{k}} e^{t_s - 1}\term{1 - \frac{1}{k}}^{1 - t_s} - e^{-1} \\
&\geq \term{2 - t_s - \frac{3}{k}}e^{t_s - 1} - e^{-1} \\
&\geq \term{2 - t_s - e^{-t_s}}e^{t_s - 1} - \frac{3}{k}\enspace,    
\end{align*}
where the first inequality holds since $\ceil{t_s \cdot k} \leq t_s \cdot k + 1$, $\alpha \leq 1$, the second inequality follows since $\alpha^{k\term{1 - t_s}} \geq e^{t_s - 1}\term{1 - \frac{1}{k}}^{1 - t_s}$, and the last inequality holds since $e^{t_s - 1} \leq 1$. Second, 
\begin{align*}
\alpha^k + \alpha^{k-1} - \frac{2k - \ceil{t_s \cdot k} - 1}{k} \alpha^{k - \floor{t_s \cdot k}} &\geq 2e^{-1} -  \frac{2k - \ceil{t_s \cdot k} - 1}{k} \alpha^{k - {t_s \cdot k}} - \frac{2e^{-1}}{k}\\
&\geq 2e^{-1} -\term{2 - t_s + \frac{1}{k}} e^{t_s - 1} - \frac{2e^{-1}}{k}\\
&\geq -e^{t_s - 1} \term{2 - t_s - 2e^{-t_s}} - \frac{1 + 2e^{-1}}{k}\enspace,    
\end{align*}
where the first inequality holds since $\alpha^{k-1} \geq \alpha^k \geq e^{-1}\term{1 - \frac{1}{k}}$, the second inequality holds since $\alpha^{k - t_s \cdot k} \geq e^{t_s - 1}$ and the last inequality holds since $t_s \in [0,1]$. Finally, it holds that $\alpha^k - \alpha^{k - \ceil{t_s \cdot k}} \geq e^{-1}\term{1 - \frac{1}{k}} - e^{t_s - 1}\term{1 - \frac{1}{k-1}} \geq -e^{t_s - 1} \term{1 - e^{-t_s}} - \frac{e^{-1}}{k}$.

Plugging all the above lower bounds into Inequality~\eqref{eq:max_inequality}yields the promised simplified guarantee that
\begin{equation*}
\begin{split}
\E{\max\{\f{A}, \f{Z}\}\mid \mathcal{E}} &\geq p_3\term{2 - t_s - e^{-t_s}}e^{t_s - 1} \f{\OPT}  - \bigO{k^{-1}}\f{\OPT}\\
&\quad + \term{\frac{p_1}{2 + \eps} -p_3 e^{t_s - 1} \term{2 - t_s - 2e^{-t_s}}} \E{\f{\OPT \cup Z} \mid \mathcal{E}}\\
&\quad \term{\frac{p_2}{1 + \eps} + \frac{p_1}{2 + \eps } -p_3e^{t_s - 1} \term{1 - e^{-t_s}}} \E{\f{\OPT \cap Z} \mid \mathcal{E}}\enspace.
\end{split}
\end{equation*}

In~\cite{buchbinder2019constrained}, It was shown that for an appropriate choice of values for $p_1$, $p_2$, $p_3$ and $t_s$ the last inequality implies $
    \E{\max\{\f{A}, \f{Z}\}\mid \mathcal{E}} \geq (c - O(k^{-1})) \cdot f(\OPT)$
for some constant $c > 0.385$. We need to observe that the event $\mathcal{E}$ happens with probability at least $1 - \eps$, and when it does not happen the set returned by the algorithm still has a non-negative value. Thus, removing the conditioning on $\mathcal{E}$ in the last inequality only affects the constant inside the big $O$ notation.

To complete the proof of the lemma, note that Line~\ref{algLine:main_run_fast_local_search} of Algorithm~\ref{alg:385_fast} requires $O_\eps\term{n + k^2}$ queries to the objective function as shown in the proof of Theorem~\ref{thm:fast_local_search}, while Line~\ref{algLine:main_run_stochastic_greedy} of Algorithm~\ref{alg:385_fast} requires $O_\eps\term{n}$ queries to the objective function as dictated by Theorem~\ref{thm:lazy_greedy_guarantee}. In total, Algorithm~\ref{alg:385_fast} requires $O_\eps\term{n + k^2}$.     
\end{proof}

\section{Experiments}
\label{sec:experiments}
In this section, we discuss three machine-learning applications: movie recommendation, image summarization, and revenue maximization. Each one of these applications necessitates maximizing a non-monotone submodular function. To emphasize the effectiveness of our suggested method from section~\ref{sec:main_result}, we empirically compare Algorithm~\ref{alg:385_fast} with the Random Greedy algorithm of Buchbinder et al.~\cite{buchbinder2014submodular}, and the Random Sampling algorithm of~\cite{buchbinder2017comparing}. These algorithms are the current state-of-the-art practical algorithms for maximizing non-monotone submodular functions.

Note that, theoretically Algorithm~\ref{alg:main_lazier_than_lazy_greedy} requires $\frac{n\cdot\ln{\eps^{-1}}}{\eps^2}$ queries to the objective function. To that end, practically, we replace Line~\ref{alg:lazier_greedy_setting_p} of Algorithm~\ref{alg:main_lazier_than_lazy_greedy} by $p \gets \frac{8}{k\cdot\eps}$. Throughout the experiments, we have chosen to use $\eps=0.1$. For each algorithm, all the reported results are averaged across $8$ executions.

\textbf{Software/Hardware}. Our algorithms were implemented in Python 3.11~\cite{python3} using mainly
“Numpy”\cite{2020NumPy-Array}, and Numba~\cite{lam2015numba}. Tests were performed on a $2.2$GHz i9-13980HX (32 cores total) machine with
$64$GB RAM.

\begin{figure}[!b]
    \centering
    \begin{subfigure}[t]{0.49\textwidth}
        \centering
        \includegraphics[width=1\textwidth]{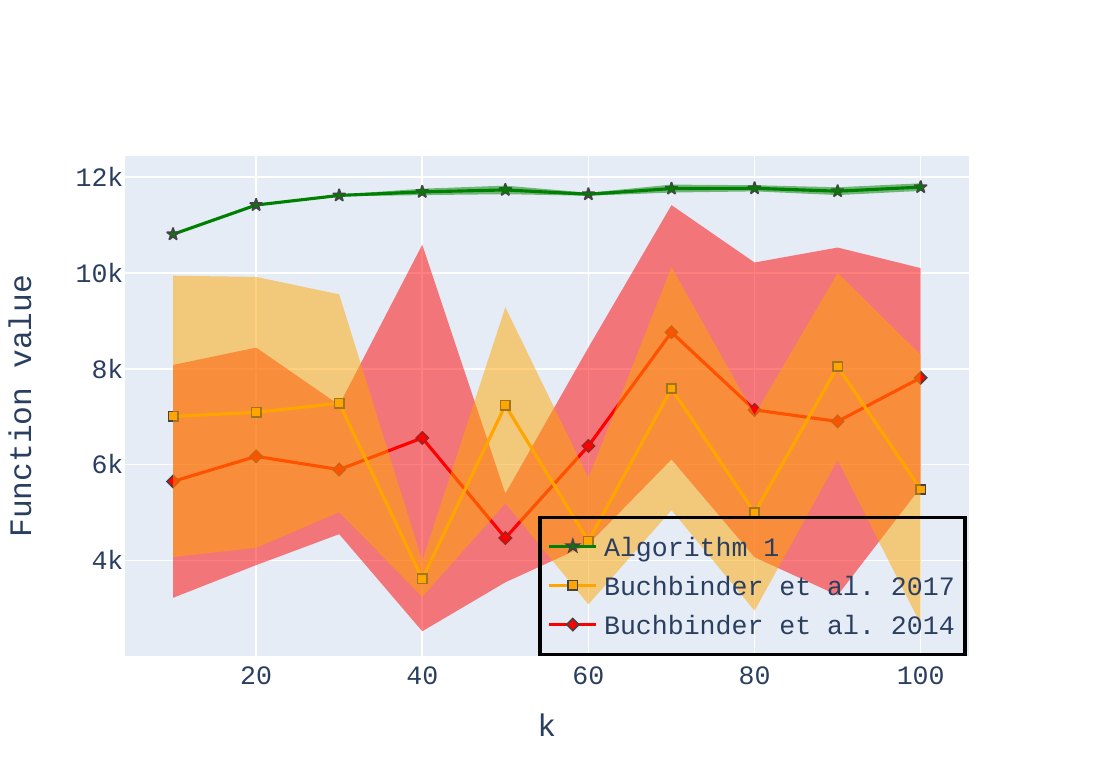}
        \caption{$\lambda := 0.75$.}
    \end{subfigure}%
    \begin{subfigure}[t]{0.49\textwidth}
        \centering
        \includegraphics[width=1\textwidth]{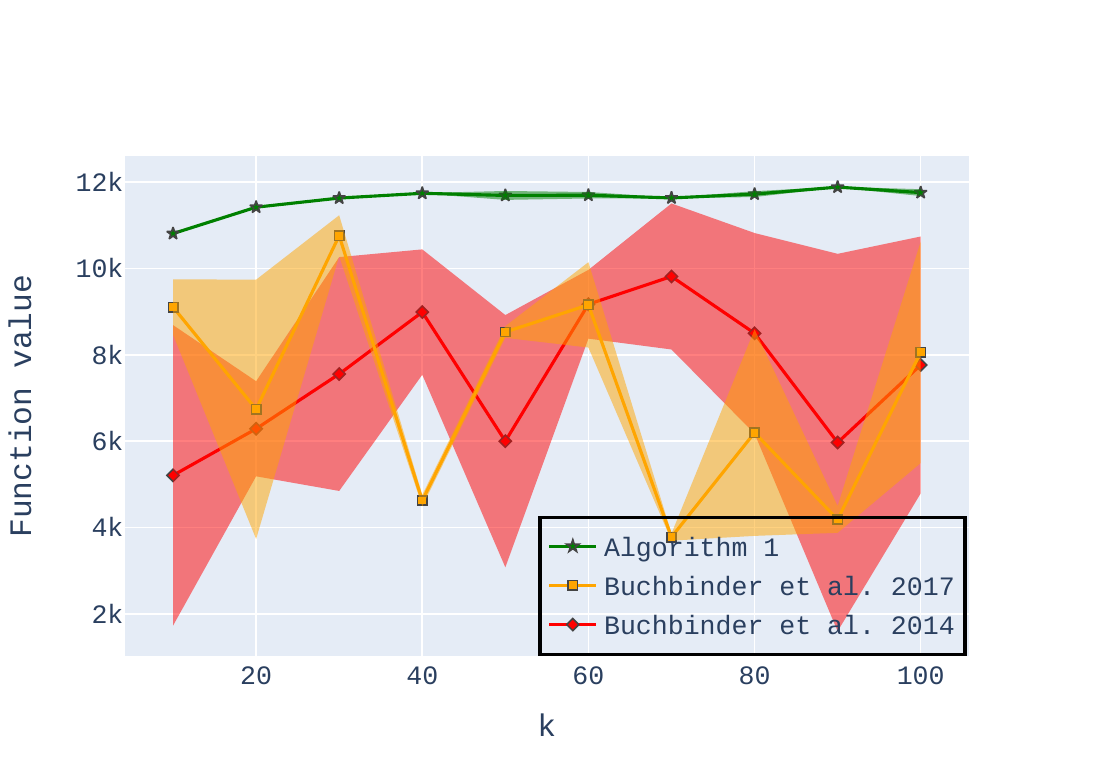}
        \caption{$\lambda := 0.55$}
    \end{subfigure}%
    \caption{Experimental results for Personalized Movie Recommendation. Each plot includes the output of our algorithm in comparison to previous state-of-the-art practical algorithms as described in the beginning of this section for various amounts of movies $k$.}
    \label{fig:movie_recommendation_exp}
\end{figure}

\subsection{Personalized movie recommendation} \label{ssc:movie_recomendation}

The first application we consider is personalized movie recommendation. Consider a movie recommendation system where each user specifies what genres she is interested in, and the system has to provide a representative subset of movies from these genres. Assume that each movie is represented by a vector consisting of users' ratings for the corresponding movie. One challenge here is that each user does not necessarily rate all the movies, hence, the vectors representing the movies do not necessarily have similar sizes. To overcome this challenge, low-rank matrix completion techniques~\cite{candes2009exact} can be performed on the matrix with missing values to obtain a complete rating matrix. Formally, given a few ratings from $k$ users to $n$ movies we obtain in this way a rating matrix $\mathbf{M}$ of size $k \times n$. Following~\cite{mualem2022using,mirzasoleiman2016fast}, to score the quality of a selected subset of movies, we use $f(S)=\sum_{u \in \mathcal{N}}\sum_{v\in S} s_{u,v}-\lambda\sum_{u\in S}\sum_{v\in S}s_{u,v}\enspace.$
Here, $\mathcal{N}$ is the set of $n$ movies, $\lambda \in [0, 1]$ is a parameter and $s_{u,v}$ denotes the similarity between movies $u$ and $v$ (the similarity $s_{u, v}$ can be calculated based on the matrix $\mathbf{M}$ in multiple ways: cosine similarity, inner product, etc). Note that the first term in the definition of $f$ captures the coverage, while the second term captures diversity. Thus, the parameter $\lambda$ denotes the importance of diversity in the returned subset which makes the function non-monotone. Note that for any $\lambda \leq 0.5$, $\f{S}$ is monotone~\cite{mualem2023resolving}. 

We followed the experimental setup of the prior works~\cite{mualem2022using,mirzasoleiman2016fast} and used a subset of movies from the MovieLens data set~\cite{harper2015movielens} which includes $10,437$ movies. Each movie in this data set is represented by a $25$ dimensional feature vector calculated using user ratings, and we used the inner product similarity to obtain the similarity values $s_{u,v}$ based on these vectors.

In this application, we fixed $\lambda$ to be $0.75,0.55$, and varied $k$. The results of this experiments are depicted by Figure~\ref{fig:movie_recommendation_exp}. Our proposed method, Algorithm~\ref{alg:385_fast}, demonstrates superior performance compared to the other methods, and demonstrate the stability of our algorithm.

\subsection{Personalized image summarization}
Consider a setting in which we get as input a collection $\mathcal{N}$ of images from $\ell$ disjoint categories (e.g., birds, dogs, cats) and the user specifies $r \in [\ell]$ categories, and then demands a subset of the images in these categories that summarize all the images of the categories. Following~\cite{mirzasoleiman2016fast} again, to evaluate a given subset of images, we use $
f(S)=\sum_{u\in \mathcal{N}}\max_{v\in S}s_{u,v} - \frac{1}{|\mathcal{N}|}\sum_{u\in S}\sum_{v\in S}s_{u,v} \enspace,$
where $s_{u,v}$ is a non-negative similarity between images $u$ and $v$.

To obtain the similarity between pair of images $u,v$, we utilized the \emph{DINO-VITB16} model~\cite{caron2021emerging} from HuggingFace [69] as the feature encoder for vision datasets. The final layer CLS token embedding output served as the feature representation. The similarity between pairs of images is them computed using the corresponding embedding vectors. Under this experiment, the following datasets were used: \begin{enumerate*}[label=(\roman*)]
    \item \emph{CIFAR10}~\cite{krizhevsky2009learning} -- A dataset of $50000$ images belonging to $10$ different classes (categories).
    \item \emph{CIFAR100}~\cite{krizhevsky2009learning} -- A dataset of $50000$ images belonging to $100$ different classes (categories).
\end{enumerate*}

For each of the datasets above, from a practical standpoint, a set of $10000$ uniformly sampled images was taken where the similarity measure between pairs of images inside the sampled set was computed using the corresponding $10,000$ embedding vectors.

The results of our experiments concerning the image summarization problem are present in Figure~\ref{fig:imageSummarization}. As depicted in the figure, our approach enjoys two distinct features -- the function values Algorithm~\ref{alg:385_fast} results in are higher than those of our competitors, and the standard deviation associated with our runs is lower than those associated with our competitors. Once again, this experiment demonstrated the stability of our algorithm in being consistent with the quality of its generated sets regardless of the randomness of Algorithm~\ref{alg:main_lazier_than_lazy_greedy}.

\begin{figure}
    \centering
    \begin{subfigure}[t]{0.49\textwidth}
        \centering
        \includegraphics[width=1\textwidth]{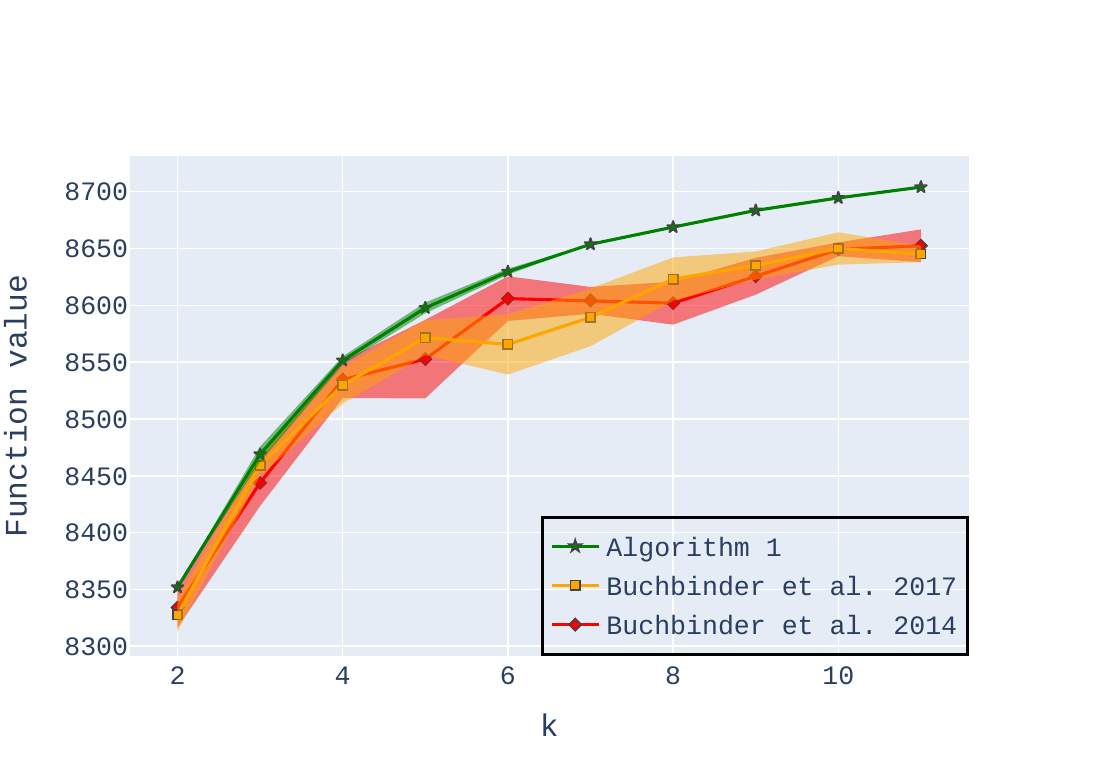}
        \caption{CIFAR10 dataset.}
    \end{subfigure}%
    \centering
    \begin{subfigure}[t]{.49\textwidth}
        \centering
         \includegraphics[width=1\textwidth]{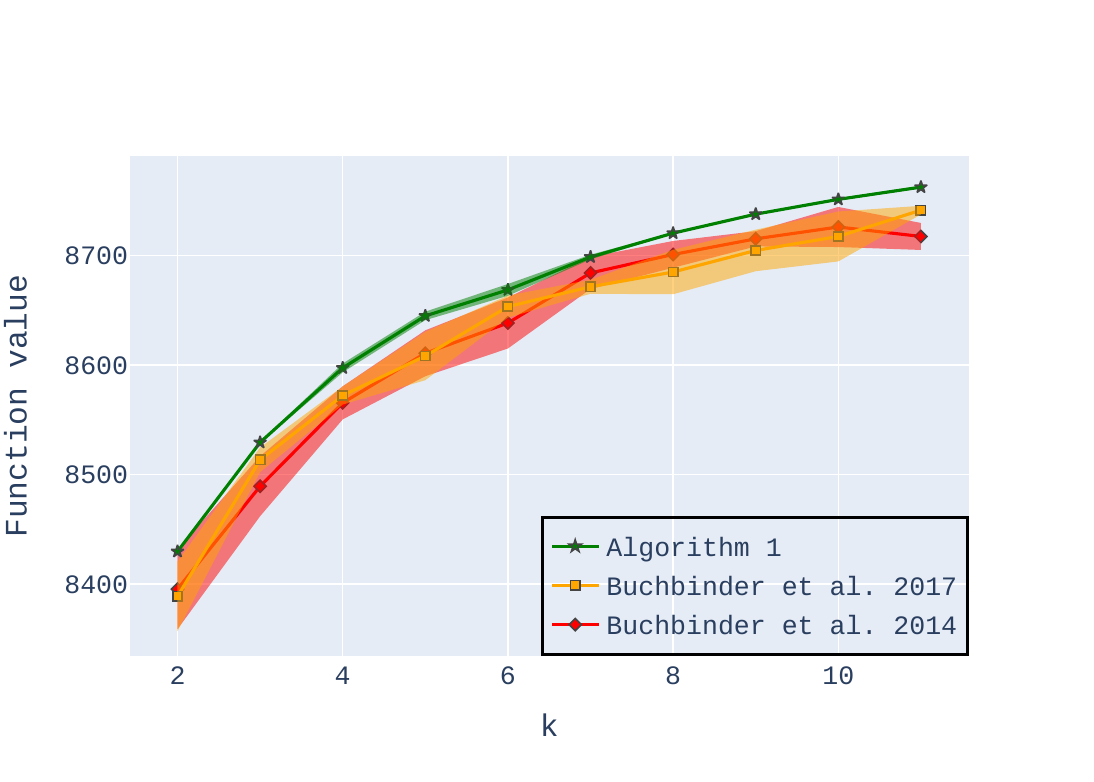}
        \caption{CIFAR100 dataset.}
    \end{subfigure}%
    \centering
    \caption{Results for a varying number of images $k$ concerning the personalized image summarization problem involving CIFAR10 and CIFAR100 datasets.}
    \label{fig:imageSummarization}
\end{figure}
\begin{figure}[t!]
    \centering
    \centering
    \begin{subfigure}[t]{0.5\textwidth}
        \centering
        \includegraphics[width=1\textwidth]{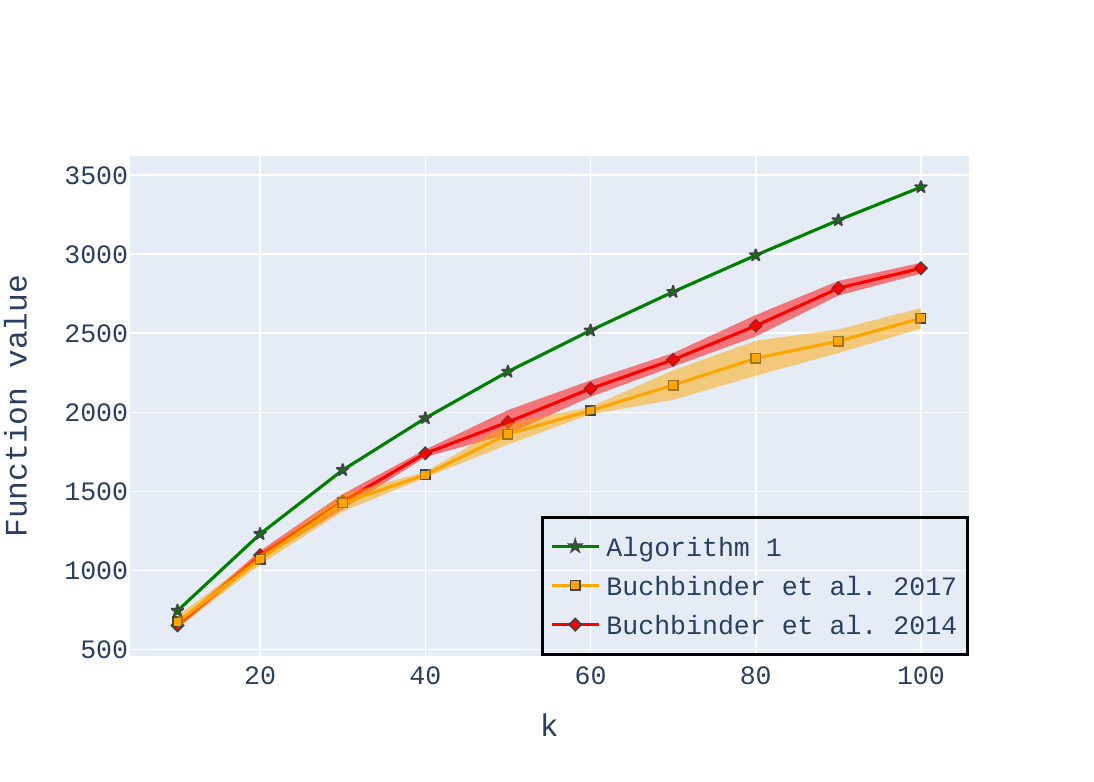}
        \caption{Advogato network dataset.}
        \label{subfigure:Advogato}
    \end{subfigure}%
    \centering
    \begin{subfigure}[t]{0.5\textwidth}
        \centering
        \includegraphics[width=1\textwidth]{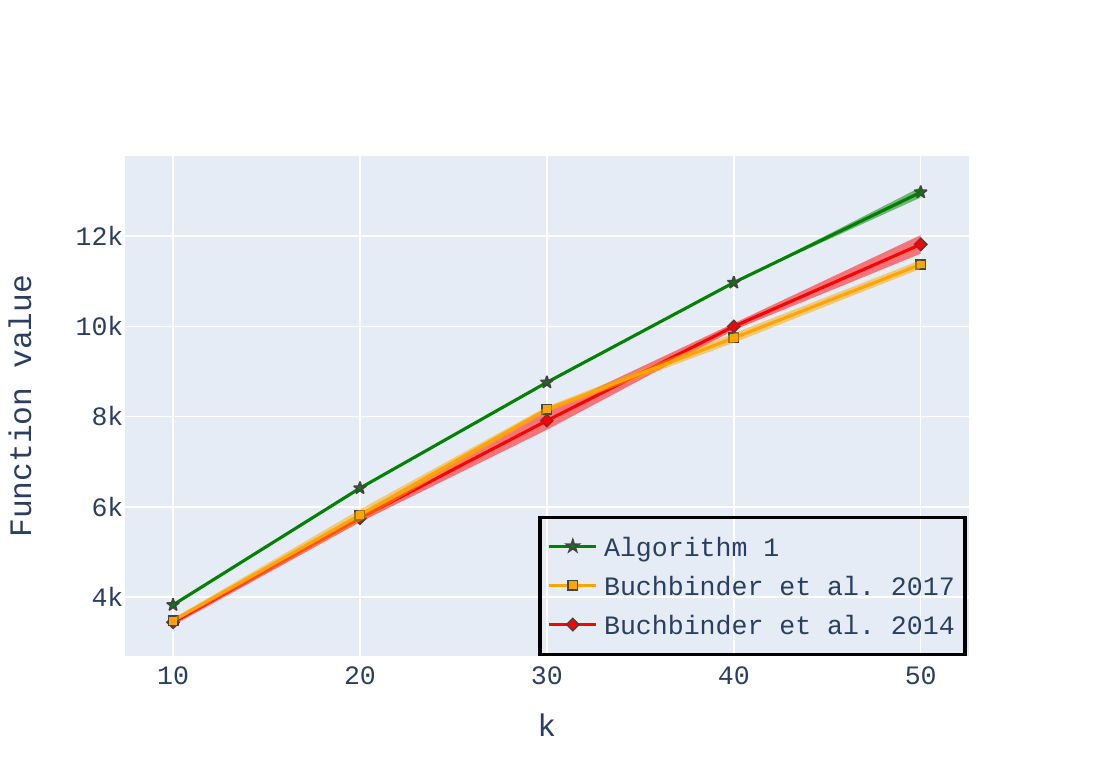}
        \caption{Facebook network dataset.}
        \label{subfigure:Facebook}
    \end{subfigure}%
    \caption{Results for a varying number of users concerning the revenue maximization problem on the Advogato and Facebook network datasets.}
    \label{fig:enter-label}
\end{figure}
\subsection{Revenue maximization}

Consider the following scenario. The objective of some company is to promote a product to users to boost revenue through the \say{word-of-mouth} effect. Specifically speaking, given a social network, we want to find a subset of $k$ users to receive a product for free in exchange for advertising it to their network neighbors, and the goal is to choose users in a manner that maximizes revenue. The problem of optimizing this objective can be formalized as follows. The input is a weighted undirected graph $G = (V, E)$ representing a social network, where $w_{ij}$ represents the weight of the edge between vertex $i$ and vertex $j$ (with $w_{ij} = 0$ if the edge $(i, j)$ is absent from the graph).

Given the set $S\subseteq V$ of users who have become advocates for the product, the expected revenue generated is proportional to the total influence of the users in $S$ on non-advocate users, formally expressed as $f(x)=\sum_{i\in S}\sum_{j\in{V\setminus S}}w_{ij}\enspace.$ It has been demonstrated that $f$ is a non-monotone submodular function~\cite{balkanski2018non}.

Figure~\ref{subfigure:Facebook} presents the function values of Algorithm~\ref{alg:385_fast}, Random Greedy~\cite{buchbinder2014submodular}, and Sample Greedy~\cite{buchbinder2017comparing}  on the Facebook network~\cite{viswanath2009evolution}. Figure~\ref{subfigure:Advogato} does the same, but using the Advogato network~\cite{massa2009dowling}.  Note that the shaded region depicts the standard deviation in each algorithm, where once again, the standard deviation associated with our runs is lower than those associated with the Random Greedy and Sample Greedy algorithms, i.e., our algorithm's performance was stable in being consistent with the quality of its generated sets.

\section{Conclusion}
In this work, we have presented a novel algorithm for submodular maximization subject to cardinality constraint that combines a practical query complexity of $O(n+k^2)$ with an approximation guarantee of $0.385$, which improves over the $\nicefrac{1}{e}$-approximation of the state-of-the-art practical algorithm. In addition to giving a theoretical analysis of our algorithms, we have demonstrated their empirical superiority (compared to practical state-of-the-art methods) in various machine learning applications. Thus, we hope to see future work that can achieve the same approximation guarantee in clean linear time.

\bibliographystyle{plain}
\bibliography{main}

\begin{thebibliography}{10}

\bibitem{balkanski2018non}
Eric Balkanski, Adam Breuer, and Yaron Singer.
\newblock Non-monotone submodular maximization in exponentially fewer
  iterations.
\newblock {\em Advances in Neural Information Processing Systems}, 31, 2018.

\bibitem{buchbinder2019constrained}
Niv Buchbinder and Moran Feldman.
\newblock Constrained submodular maximization via a nonsymmetric technique.
\newblock {\em Mathematics of Operations Research}, 44(3):988--1005, 2019.

\bibitem{buchbinder2023constrained}
Niv Buchbinder and Moran Feldman.
\newblock Constrained submodular maximization via new bounds for dr-submodular
  functions.
\newblock {\em arXiv preprint arXiv:2311.01129}, 2023.

\bibitem{buchbinder2014submodular}
Niv Buchbinder, Moran Feldman, Joseph Naor, and Roy Schwartz.
\newblock Submodular maximization with cardinality constraints.
\newblock In {\em Proceedings of the twenty-fifth annual ACM-SIAM symposium on
  Discrete algorithms}, pages 1433--1452. SIAM, 2014.

\bibitem{buchbinder2017comparing}
Niv Buchbinder, Moran Feldman, and Roy Schwartz.
\newblock Comparing apples and oranges: Query trade-off in submodular
  maximization.
\newblock {\em Mathematics of Operations Research}, 42(2):308--329, 2017.

\bibitem{calinescu2011calinescu}
Gruia C{\u{a}}linescu, Chandra Chekuri, Martin P{\'{a}}l, and Jan
  Vondr{\'{a}}k.
\newblock Maximizing a monotone submodular function subject to a matroid
  constraint.
\newblock {\em {SIAM} J. Comput.}, 40(6):1740--1766, 2011.

\bibitem{candes2009exact}
Emmanuel~J Cand{\`e}s and Benjamin Recht.
\newblock Exact matrix completion via convex optimization.
\newblock {\em Foundations of Computational mathematics}, 9(6):717--772, 2009.

\bibitem{caron2021emerging}
Mathilde Caron, Hugo Touvron, Ishan Misra, Herv{\'e} J{\'e}gou, Julien Mairal,
  Piotr Bojanowski, and Armand Joulin.
\newblock Emerging properties in self-supervised vision transformers.
\newblock In {\em Proceedings of the IEEE/CVF international conference on
  computer vision}, pages 9650--9660, 2021.

\bibitem{chekuri2014submodular}
Chandra Chekuri, Jan Vondr{\'{a}}k, and Rico Zenklusen.
\newblock Submodular function maximization via the multilinear relaxation and
  contention resolution schemes.
\newblock {\em {SIAM} J. Comput.}, 43(6):1831--1879, 2014.

\bibitem{chen2024guided}
Yixin Chen, Ankur Nath, Chunli Peng, and Alan Kuhnle.
\newblock Guided combinatorial algorithms for submodular maximization.
\newblock {\em arXiv preprint arXiv:2405.05202}, 2024.

\bibitem{das2011submodular}
Abhimanyu Das and David Kempe.
\newblock Submodular meets spectral: greedy algorithms for subset selection,
  sparse approximation and dictionary selection.
\newblock In {\em ICML}, pages 1057--1064, 2011.

\bibitem{elenberg2017streaming}
Ethan~R. Elenberg, Alexandros~G. Dimakis, Moran Feldman, and Amin Karbasi.
\newblock Streaming weak submodularity: interpreting neural networks on the
  fly.
\newblock In {\em NeurIPS}, pages 4047--4057, 2017.

\bibitem{ene2016constrained}
Alina Ene and Huy~L Nguyen.
\newblock Constrained submodular maximization: Beyond 1/e.
\newblock In {\em 2016 IEEE 57th Annual Symposium on Foundations of Computer
  Science (FOCS)}, pages 248--257. IEEE, 2016.

\bibitem{feige2011maximizing}
Uriel Feige, Vahab~S. Mirrokni, and Jan Vondr{\'{a}}k.
\newblock Maximizing non-monotone submodular functions.
\newblock {\em {SIAM} J. Comput.}, 40(4):1133--1153, 2011.

\bibitem{han2020deterministic}
Kai Han, Shuang Cui, Benwei Wu, et~al.
\newblock Deterministic approximation for submodular maximization over a
  matroid in nearly linear time.
\newblock {\em Advances in Neural Information Processing Systems}, 33:430--441,
  2020.

\bibitem{harper2015movielens}
F.~Maxwell Harper and Joseph~A. Konstan.
\newblock The movielens datasets: History and context.
\newblock {\em Acm transactions on interactive intelligent systems ({TiiS})},
  5(4):1--19, 2015.

\bibitem{2020NumPy-Array}
Charles~R. Harris, K.~Jarrod Millman, Stéfan~J van~der Walt, Ralf Gommers,
  Pauli Virtanen, David Cournapeau, Eric Wieser, Julian Taylor, Sebastian Berg,
  Nathaniel~J. Smith, Robert Kern, Matti Picus, Stephan Hoyer, Marten~H. van
  Kerkwijk, Matthew Brett, Allan Haldane, Jaime Fernández~del Río, Mark
  Wiebe, Pearu Peterson, Pierre Gérard-Marchant, Kevin Sheppard, Tyler Reddy,
  Warren Weckesser, Hameer Abbasi, Christoph Gohlke, and Travis~E. Oliphant.
\newblock Array programming with {NumPy}.
\newblock {\em Nature}, 585:357–362, 2020.

\bibitem{huang2022multi}
Chien-Chung Huang and Naonori Kakimura.
\newblock Multi-pass streaming algorithms for monotone submodular function
  maximization.
\newblock {\em Theory of Computing Systems}, 66(1):354--394, 2022.

\bibitem{krizhevsky2009learning}
Alex Krizhevsky, Geoffrey Hinton, et~al.
\newblock Learning multiple layers of features from tiny images.
\newblock 2009.

\bibitem{kuhnle2021quick}
Alan Kuhnle.
\newblock Quick streaming algorithms for maximization of monotone submodular
  functions in linear time.
\newblock In {\em International Conference on Artificial Intelligence and
  Statistics}, pages 1360--1368. PMLR, 2021.

\bibitem{lam2015numba}
Siu~Kwan Lam, Antoine Pitrou, and Stanley Seibert.
\newblock Numba: A llvm-based python jit compiler.
\newblock In {\em Proceedings of the Second Workshop on the LLVM Compiler
  Infrastructure in HPC}, pages 1--6, 2015.

\bibitem{lee2009nonmonotone}
Jon Lee, Vahab~S. Mirrokni, Viswanath Nagarajan, and Maxim Sviridenko.
\newblock Non-monotone submodular maximization under matroid and knapsack
  constraints.
\newblock In Michael Mitzenmacher, editor, {\em Proceedings of the 41st Annual
  {ACM} Symposium on Theory of Computing ({STOC})}, pages 323--332. {ACM},
  2009.

\bibitem{lei2019discrete}
Qi~Lei, Lingfei Wu, Pin{-}Yu Chen, Alex Dimakis, Inderjit~S. Dhillon, and
  Michael~J. Witbrock.
\newblock Discrete adversarial attacks and submodular optimization with
  applications to text classification.
\newblock In Ameet Talwalkar, Virginia Smith, and Matei Zaharia, editors, {\em
  MLSys}. mlsys.org, 2019.

\bibitem{li2022submodular}
Wenxin Li, Moran Feldman, Ehsan Kazemi, and Amin Karbasi.
\newblock Submodular maximization in clean linear time.
\newblock {\em Advances in Neural Information Processing Systems},
  35:17473--17487, 2022.

\bibitem{lovasz1983submodular}
L\'{a}szl\'{o} Lov\'{a}sz.
\newblock Submodular functions and convexity.
\newblock In A.~Bachem, M.~Gr\"{o}tschel, and B.~Korte, editors, {\em
  Mathematical Programming: the State of the Art}, pages 235--257. Springer,
  1983.

\bibitem{massa2009dowling}
Paolo Massa, Martino Salvetti, and Danilo Tomasoni.
\newblock Bowling alone and trust decline in social network sites.
\newblock In {\em {IEEE} International Conference on Dependable, Autonomic and
  Secure Computing (DASC)}, pages 658--663. {IEEE} Computer Society, 2009.

\bibitem{mirzasoleiman2016fast}
Baharan Mirzasoleiman, Ashwinkumar Badanidiyuru, and Amin Karbasi.
\newblock Fast constrained submodular maximization: Personalized data
  summarization.
\newblock In {\em ICML}, pages 1358--1367. PMLR, 2016.

\bibitem{mirzasoleiman2015lazier}
Baharan Mirzasoleiman, Ashwinkumar Badanidiyuru, Amin Karbasi, Jan Vondr{\'a}k,
  and Andreas Krause.
\newblock Lazier than lazy greedy.
\newblock In {\em Proceedings of the AAAI Conference on Artificial
  Intelligence}, volume~29, 2015.

\bibitem{mitrovic2018data}
Marko Mitrovic, Ehsan Kazemi, Morteza Zadimoghaddam, and Amin Karbasi.
\newblock Data summarization at scale: A two-stage submodular approach.
\newblock In {\em ICML}, pages 3596--3605. PMLR, 2018.

\bibitem{mualem2022using}
Loay Mualem and Moran Feldman.
\newblock Using partial monotonicity in submodular maximization.
\newblock {\em Advances in Neural Information Processing Systems},
  35:2723--2736, 2022.

\bibitem{mualem2023resolving}
Loay Mualem and Moran Feldman.
\newblock Resolving the approximability of offline and online non-monotone
  dr-submodular maximization over general convex sets.
\newblock In {\em International Conference on Artificial Intelligence and
  Statistics}, pages 2542--2564. PMLR, 2023.

\bibitem{mualem2024bridging}
Loay Mualem, Murad Tukan, and Moran Fledman.
\newblock Bridging the gap between general and down-closed convex sets in
  submodular maximization.
\newblock {\em arXiv preprint arXiv:2401.09251}, 2024.

\bibitem{mualem2024submodular}
Loay~Raed Mualem, Ethan~R Elenberg, Moran Feldman, and Amin Karbasi.
\newblock Submodular minimax optimization: Finding effective sets.
\newblock In {\em International Conference on Artificial Intelligence and
  Statistics}, pages 1081--1089. PMLR, 2024.

\bibitem{nemhauser1978best}
George~L Nemhauser and Laurence~A Wolsey.
\newblock Best algorithms for approximating the maximum of a submodular set
  function.
\newblock {\em Mathematics of operations research}, 3(3):177--188, 1978.

\bibitem{nemhauser1978analysis}
George~L Nemhauser, Laurence~A Wolsey, and Marshall~L Fisher.
\newblock An analysis of approximations for maximizing submodular set
  functions—i.
\newblock {\em Mathematical programming}, 14:265--294, 1978.

\bibitem{norouzi2018beyond}
Ashkan Norouzi-Fard, Jakub Tarnawski, Slobodan Mitrovic, Amir Zandieh,
  Aidasadat Mousavifar, and Ola Svensson.
\newblock Beyond 1/2-approximation for submodular maximization on massive data
  streams.
\newblock In {\em ICML}, pages 3829--3838. PMLR, 2018.

\bibitem{qi2022maximizing}
Benjamin Qi.
\newblock On maximizing sums of non-monotone submodular and linear functions.
\newblock In Sang~Won Bae and Heejin Park, editors, {\em International
  Symposium on Algorithms and Computation (ISAAC)}, volume 248 of {\em LIPIcs},
  pages 41:1--41:16. Schloss Dagstuhl - Leibniz-Zentrum f{\"{u}}r Informatik,
  2022.

\bibitem{tukan2023orbslam3}
Murad Tukan, Fares Fares, Yotam Grufinkle, Ido Talmor, Loay Mualem, Vladimir
  Braverman, and Dan Feldman.
\newblock Orbslam3-enhanced autonomous toy drones: Pioneering indoor
  exploration.
\newblock {\em arXiv preprint arXiv:2312.13385}, 2023.

\bibitem{python3}
Guido Van~Rossum and Fred~L. Drake.
\newblock {\em Python 3 Reference Manual}.
\newblock CreateSpace, Scotts Valley, CA, 2009.

\bibitem{viswanath2009evolution}
Bimal Viswanath, Alan Mislove, Meeyoung Cha, and Krishna~P Gummadi.
\newblock On the evolution of user interaction in facebook.
\newblock In {\em Proceedings of the 2nd ACM workshop on Online social
  networks}, pages 37--42, 2009.

\bibitem{vondrak2013symmetry}
Jan Vondr{\'{a}}k.
\newblock Symmetry and approximability of submodular maximization problems.
\newblock {\em {SIAM} J. Comput.}, 42(1):265--304, 2013.

\bibitem{zhou2022risk}
Lifeng Zhou and Pratap Tokekar.
\newblock Risk-aware submodular optimization for multirobot coordination.
\newblock {\em IEEE Transactions on Robotics}, 38(5):3064--3084, 2022.

\end{thebibliography}

\appendix
\section{A warmup version of our algorithm}
\label{sec:warmup_methods}
In this section, we present a simpler version of our algorithm with the same general structure but exclude the speedup techniques used to obtain our main result. We present and analyze this algorithm. Inspired by Buchbinder et al.~\cite{buchbinder2019constrained}, and similar to Algorithm~\ref{alg:385_fast}, Algorithm~\ref{alg:main_algorithm_simple} comprises three steps: (i) searching for a good initial solution that guarantees a constant approximation to the optimal set. This is accomplished by running the Twin Greedy algorithm of~\cite{han2020deterministic}, (ii) finding an (approximate) local search optimum set $Z$ using a local search method. In this simple version, we use the classical local search algorithm which requires $O(nk^2)$ queries to the objective function. (iii) Lastly, using the set $Z$ obtained from the previous step, we attempt to improve the solution using the Random Greedy Algorithm suggested by Buchbinder et al.~\cite{buchbinder2014submodular}.

\subsection{Local search}
In what follows, we present a simple local search algorithm, which is the algorithm used to implement the first two steps of Algorithm~\ref{alg:main_algorithm_simple}. 

Algorithm~\ref{alg:main_algorithm_simple} begins by finding an initial solution $S_0$ using the Random Greedy algorithm from~\cite{buchbinder2014submodular}. The algorithm then proceeds as follows: (i) If $|S| < k$, it checks for an element $u$ such that adding $u$ to $S$ increases the function value by at least $(1 + \frac{\eps}{k})f(S)$. If such an element is found, it is added to $S$. (ii) If $|S| = k$, it looks for two elements $u \in \mathcal{N} \setminus S$ and $v \in S$ such that swapping $u$ and $v$ (i.e., removing $v$ from $S$ and adding $u$ to $S$) increases the function value by at least $(1 + \frac{\eps}{k})f(S)$. If such elements exist, the algorithm performs the swap. (iii) If $|S| > 0$ and no elements satisfy the previous two conditions, it checks for an element $v$ such that removing $v$ from $S$ increases the function value by $(1 + \frac{\eps}{k})f(S)$. If such an element is found, it is removed from $S$.

The algorithm continues to search for elements satisfying any of the above three conditions until no such elements exist. At this point, the algorithm terminates and returns the set $S$ as the final output.

\begin{algorithm}[!htb]
    \SetKwInOut{Input}{input}
    \SetKwInOut{Output}{output}

    \Input{A positive integer $k \geq 1$, a submodular function $f$, and an error parameter $\eps \in (0,1)$.}
    \Output{A set $S \subseteq \mathcal{N}$}    
    
    Initialize $S_0$ to be a feasible solution guaranteeing $c$-approximation for the problem for some constant $c\in(0,1]$\label{alg:LocalSearchInit}\\
    \While{$\mathrm{true}$}{

        \uIf{$\exists u \in \mathcal{N} \setminus S$ such that $\f{S + u} \geq \term{1 + \frac{\eps}{k}}\f{S}$ and $\abs{S} < k$}{
            
            $S \gets S + u$.
        }\uElseIf{$\exists u \in \mathcal{N} \setminus S, v \in S$ such that $\f{S + u - v} \geq \term{1 + \frac{\eps}{k}}\f{S}$ and $\abs{S} = k$}{
            
            $S \gets S - v + u$.
        }\uElseIf{$\exists v \in S$ such that $\f{S - v} \geq \term{1 + \frac{\eps}{k}}\f{S}$  and $\abs{S} > 0$}{
            $S \gets S - v$.
        }\uElse{
            \Return S
        }
    }
    
    \caption{$\localsearch\term{k,f}$}
    \label{alg:local_search}
\end{algorithm}

Let $\OPT$ be an optimal solution. The properties of Algorithm~\ref{alg:local_search} are formally established in Theorem~\ref{thm:localSearchGuarantee}. To prove this theorem, we first need the following lemma.

\begin{restatable}{lemma}{localsearchlem}
\label{lem:localsearch}
Given a positive integer $k$, a submodular function $f$ and an error parameter $\eps \in (0,1)$, Algorithm~\ref{alg:local_search} returns a set $S \subseteq \mathcal{N}$ of size at max $k$ such that
\begin{equation*}
\f{S} \geq \frac{\f{S \cup \OPT} + \f{S \cap \OPT}}{2 + \eps} \qquad \text{and} \qquad \frac{\f{S \cap \OPT}}{1 + \eps}\enspace,
\end{equation*}

while requiring $O_\eps\term{nk^2}$ queries to the objective function.
\end{restatable}

\begin{proof}
First, let $E_{-}$, $E_{\pm}$, and $E_{+}$ be defined as follows.
\begin{itemize}
    \item $E_{-}$ be the event of $\exists v \in S$ such that $\f{S - v} \geq \term{1 + \frac{\eps}{k}}\f{S}$ and $\abs{S} > 0$.
    \item $E_{\pm}$ be the event of $\exists u \in \mathcal{N} \setminus S, v \in S$ such that $\f{S + u - v} \geq \term{1 + \frac{\eps}{k}}\f{S}$ and $\abs{S} = k$.
    \item $E_{+}$ be the event of $\exists u \in \mathcal{N} \setminus S$ such that $\f{S + u}\geq \term{1 + \frac{\eps}{k}}\f{S}$ and $\abs{S} < k$.
\end{itemize} 
  
Let $S$ be a subset in $\mathcal{N}$ of size at $k$ be the output of  Algorithm~\ref{alg:local_search}. By construction of the algorithm, $S$ ensured that each of the events $E_+, E_-$ and $E_{\pm}$ did not occur. 

To obtain the theoretical guarantees associated with the set $S$, we inspect the implications arising from each of the events above not occurring.

\paragraph{Handling the case where $E_{+}$ is not  holding.} Assume that $\abs{S} < k$. This case implies that for every $u \in \mathcal{N} \setminus S$,  it holds that $\term{1 + \frac{\eps}{k}}\f{S} > \f{S + u}$.

Summing the above term across every element in $\OPT \setminus S$ yields
\begin{equation*}
\begin{split}
\term{1 + \frac{\eps}{k}}\f{S} &\geq \f{S} + \frac{1}{\abs{\OPT \setminus S}} \sum\limits_{u \in \OPT \setminus S} \sterm{\f{S+ u} - \f{S}} \\
&\geq \f{S} + \frac{1}{\abs{\OPT \setminus S}} \term{\f{S \cup \OPT} - \f{S}}\enspace,
\end{split}
\end{equation*}
where the first inequality holds from the assumptions of this case, and the second inequality follows from the submodularity of $\f{\cdot}$. By simple rearrangement, we obtain that 
\begin{equation}
\begin{split}
\f{S} \geq \frac{1}{\frac{\eps}{k}\abs{\OPT \setminus S} + 1} \f{S \cup \OPT} \geq \frac{1}{1 + \eps} \f{S \cup \OPT}\enspace,
\end{split}
\end{equation}
where the last inequality holds since $\abs{\OPT \setminus S} \leq \abs{\OPT} \leq k$.

\paragraph{Handling the case where $E_{-}$ is not holding.} Similarly, by the case's assumption it holds that for every $v \in S$, $\term{1 + \frac{\eps}{k}}\f{S} > \f{S -v}$.

Summing the above term across every element in $S \setminus \OPT$, yields 
\begin{equation*}
\begin{split}
\term{1 + \frac{\eps}{k}}\f{S} &\geq \f{S} + \frac{1}{\abs{S \setminus \OPT}} \sum\limits_{v \in S \setminus \OPT} \sterm{\f{S - v} - \f{S}} \\
&\geq \f{S} + \frac{1}{\abs{S \setminus \OPT}} \term{\f{S \cap \OPT} - \f{S}}\enspace,
\end{split}
\end{equation*}
where the last inequality holds by submodularity of $\f{\cdot}$. By simple rearrangement, we obtain that 
\begin{equation}
\label{eq:handling_case_E-}
\begin{split}
\f{S} \geq \frac{1}{\frac{\eps}{k}\abs{S \setminus \OPT} + 1} \f{S \cap \OPT} \geq \frac{1}{1 + \eps} \f{S \cap \OPT}\enspace,
\end{split}
\end{equation}
where the last inequality holds since $\abs{S \setminus \OPT} \leq \abs{S} \leq k$.

\paragraph{Handling the case where $E_{\pm}$ is not  holding.} This case implies that for every pair $u,v$ where $u \in \mathcal{N} \setminus S$ and $v \in S$, $\term{1 + \frac{\eps}{k}}\f{S} \geq \f{S + u - v}$.

Summing the above term for every $v \in S \setminus \OPT$ and $u \in \OPT \setminus S$, yields that
\begin{equation}
\label{eq:last_case}
\term{1 + \frac{\eps}{k}} \f{S} \geq \frac{1}{\abs{S \setminus \OPT} \cdot \abs{\OPT \setminus S}} \sum\limits_{u \in \OPT \setminus S} \sum\limits_{v \in S \setminus \OPT} \f{S + u - v}\enspace.    
\end{equation} 

Observe that
\begin{equation*}
\begin{split}
&\sum\limits_{u \in \OPT \setminus S} \sum\limits_{v \in S \setminus \OPT} \f{S +u - v} - \f{S} = \sum\limits_{u \in \OPT \setminus S} \sum\limits_{v \in S \setminus \OPT} \f{S + u - v} - \f{S} \\
&\quad = \sum\limits_{u \in \OPT \setminus S} \sum\limits_{v \in S \setminus \OPT} \f{S + u - v} - \f{S- v} + \f{S - v} - \f{S}\enspace,
\end{split}
\end{equation*}
by submodularity it holds that
\begin{equation*}
\begin{split}
&\frac{1}{\abs{S \setminus \OPT}\abs{\OPT \setminus S}} \sum\limits_{u \in \OPT \setminus S} \sum\limits_{v \in S \setminus \OPT} \f{S \cup \br{u} \setminus \br{v}} - \f{S} \\
&\quad \geq \overbrace{\frac{\sum\limits_{u \in \OPT \setminus S} \sum\limits_{v \in S \setminus \OPT} \f{S \cup \br{u} \setminus \br{v}} - \f{S \setminus \br{v}}}{\abs{S \setminus \OPT} \cdot \abs{\OPT \setminus S}}}^\mathbb{A} \\
&\quad\quad + \underbrace{\frac{\sum\limits_{u \in \OPT \setminus S} \sum\limits_{v \in S \setminus \OPT} \f{S \setminus \br{v}} - \f{S}}{\abs{S \setminus \OPT} \cdot \abs{\OPT \setminus S}}}_{\mathbb{B}}\enspace.
\end{split}
\end{equation*}

\textbf{Bounding $\mathbb{A}$.} Note by submodularity of $\f{\cdot}$,
\begin{equation*}
\begin{split}
\sum\limits_{u \in \OPT \setminus S} \f{S + u - v} - \f{S - v} &\geq \sum\limits_{u \in \OPT \setminus S} \f{S + u} - \f{S} \\
&\geq \f{S \cup \OPT} - \f{S}\enspace.    
\end{split}
\end{equation*}
Hence, 
\begin{equation*}
\begin{split}
\frac{\sum\limits_{u \in \OPT \setminus S} \sum\limits_{v \in S \setminus \OPT} \f{S + u - v} - \f{S - v}}{\abs{S \setminus \OPT} \cdot \abs{\OPT \setminus S}} &\geq \frac{\sum\limits_{u \in \OPT \setminus S} \f{S + u } - \f{S}}{\abs{\OPT \setminus S}} \\
&\geq \frac{\f{S \cup \OPT} - \f{S}}{\abs{\OPT \setminus S}}\enspace.
\end{split}
\end{equation*}

\textbf{Bounding $\mathbb{B}$.} Using~\eqref{eq:handling_case_E-} gives
\begin{equation*}
\begin{split}
\frac{\sum\limits_{u \in \OPT \setminus S} \sum\limits_{v \in S \setminus \OPT} \f{S - v} - \f{S}}{\abs{S \setminus \OPT} \cdot \abs{\OPT \setminus S}} = \frac{\sum\limits_{v \in S \setminus \OPT} \f{S - v} - \f{S}}{\abs{S \setminus \OPT}} \geq \frac{\f{S \cap \OPT} - \f{S}}{\abs{S \setminus \OPT}}\enspace.    
\end{split}
\end{equation*}

Combining all of the above yields that 
\begin{equation*}
\begin{split}
\frac{\eps}{k}\f{S} \geq \frac{\f{S \cap \OPT} - \f{S}}{\abs{S \setminus \OPT}} + \frac{\f{S \cup \OPT} - \f{S}}{\abs{\OPT \setminus S}}\enspace,
\end{split}
\end{equation*}
whereby simple rearrangement, we obtain that
\begin{equation*}
\begin{split}
\f{S} &\geq \frac{\f{S \cap \OPT}}{1 + \frac{\eps}{k}\abs{S \setminus \OPT} + \frac{\abs{S \setminus \OPT}}{\abs{\OPT  \setminus S}}} +  \frac{\f{S \cup \OPT}}{1 + \frac{\eps}{k}\abs{\OPT \setminus S} + \frac{\abs{\OPT \setminus S}}{\abs{S \setminus \OPT}}} \\
&\geq \frac{\f{S \cup \OPT} + \f{S \cap \OPT}}{2 + \eps},
\end{split}
\end{equation*}
where the last inequality holds since $\abs{\OPT \setminus S} = \abs{S \setminus \OPT} \leq k$.
\end{proof}

We are now ready to present and prove the properties of Algorithm~\ref{alg:local_search} which are described in the Theorem below.
\begin{theorem}
\label{thm:localSearchGuarantee}
Given a positive integer $k$, a submodular function $f$, and an error parameter $\eps \in (0,1)$, there exists an algorithm that outputs a set $S \subseteq \mathcal{N}$ of size at max $k$ obeying 
\[
\f{S} \geq \frac{\f{S \cup \OPT} + \f{S \cap \OPT}}{2 + \eps} \qquad \text{and} \qquad \frac{\f{S \cap \OPT}}{1 + \eps}\enspace.
\]
Requiring  $O_\eps\term{nk^2}$ queries to the objective function.

\end{theorem}

\begin{proof}
First, we note that $S$ can be initialized at Line~\ref{alg:LocalSearchInit} of Algorithm~\ref{alg:local_search} by~\cite{han2020deterministic}, which gives an approximation $c = \term{\nicefrac{1}{4} - \eps}$, while requiring $O\term{\frac{n}{\eps}\ln{\frac{k}{\eps}}}$ queries to the objective function.

To obtain the lower bounds in Theorem~\ref{thm:localSearchGuarantee}, we use Lemma~\ref{lem:localsearch}. To finalize the proof, we bound the number of iterations $L$ needed to return a set $S$ obeying the lower bounds in Theorem~\ref{thm:localSearchGuarantee}. Observe that the function value of the set $S$ in Algorithm~\ref{alg:local_search} increases by a multiplicative factor of $\frac{\eps}{k}$ at each iteration of the while loop until no change is made. To this end, and since $S$ is initialized with a solution that is guaranteed to be lower bounded by $\term{\frac{1}{4} - \eps} \f{\OPT}$, we note that
\[
\f{OPT} \geq \f{S} \geq \term{1 + \frac{\eps}{k}}^L c\f{\OPT}\enspace,
\]
resulting in
\[
\frac{\ln\term{\frac{1}{c}}}{\ln\term{1 +\frac{\eps}{k}}} \geq L\enspace,
\]
where this gives that $\frac{2\ln{\frac{1}{c}}}{\eps}k \geq L$, implying that Algorithm~\ref{alg:local_search} requires $O\term{nk^2\eps^{-1}}$ queries to the objective function. 
\end{proof}

\subsection{Guided Random Greedy}
In this section, we present the Guided Random Greedy algorithm. The algorithm starts with an empty set, and adds one element in each iteration until it returns the final solution after $k$ iterations. In its first $\lceil k \cdot t_s \rceil $ iterations, the algorithm ignores the elements of $Z$, and in the rest of the iterations, it considers all elements. However, except for this difference, the behavior of the algorithm in all iterations is very similar. Specifically, in each iteration $i$ the algorithm does the following two steps. In Step (i) the algorithm finds a subset $M_i$ of size $k$ maximizing the sum of marginal gains of the elements $u\in M_i$ with respect to the current solution $S_{i-1}$. In step $(ii)$, the algorithm chooses a random element from $M_i$ and adds it to the solution.

\begin{algorithm}[htb!]
    \SetKwInOut{Input}{input}
    \SetKwInOut{Output}{output}

    \Input{A set $Z \subseteq \mathcal{N}$, a positive integer $k \geq 1$, a submodular function $f$ and  a flip point $t_s \in [0,1]$}
    \Output{A set $S \subseteq \mathcal{N}$}

    Initialize $S_0 \gets \emptyset$ \\
    \For{$i = 1$ to $\ceil{k \cdot t_s}$}{
        Let $M_i \subseteq \mathcal{N} \setminus \term{S_{i-1} \cup Z}$ be a subset of size $k$ maximizing $\sum\limits_{u \in M_i} \f{S_{i-1} + u} - \f{S_{i-1}}$.\\
        Let $u_i$ be a uniformly random element from $M_i$. \\
        $S_i \gets S_{i-1} + u_i$
    }
   \For{$i = \ceil{k \cdot t_s} + 1$ to $k$}{
        Let $M_i \subseteq \mathcal{N} \setminus S_{i-1}$ be a subset of size $k$ maximizing $\sum_{u \in M_i} \f{S_{i-1} \cup u} - \f{S_{i-1}}$.\\
        Let $u_i$ be a uniformly random element from $M_i$. \\
        $S_i \gets S_{i-1} + u_i$
    }    
    \caption{Guided Random Greedy}
    \label{alg:main_random_greedy}
\end{algorithm}
To prove the main Theorem of this algorithm (Theorem~\ref{thm:random_greedy_guarantee}), we first provide and prove two lemmas which will be useful in the proof of Theorem~\ref{thm:random_greedy_guarantee}. 

First, we define the Lov\'{a}sz extension of a submodular function $f \colon 2^{\mathcal{N}} \to \REAL$, which will be useful in the proof of Lemma~\ref{lem:S_i_union_A}.

The Lov\'{a}sz extension of $f$ is a function $\hat{f}\colon [0, 1]^\mathcal{N} \to \mathbb{R}$ defined as follows. For every vector $x \in [0, 1]^\mathcal{N}$,
\[
	\hat{f}(x)
	=
	\int_0^1 f(T_\lambda(x)) d\lambda
	\enspace,
\]
where $T_\lambda(x) \triangleq \{u \in \mathcal{N} \mid x_u \geq \lambda\}$. The Lov\'{a}sz extension of a submodular function is known to be convex. More important for us is the following known lemma regarding this extension. This lemma stems from an equality, proved by Lov\'{a}sz~\cite{lovasz1983submodular}, between the Lov\'{a}sz extension of a submodular function and another extension known as the convex closure.
\begin{lemma} \label{lem:lovasz}
Let $f\colon 2^{\mathcal{N}} \to \mathbb{R}$ be a submodular function, and let $\hat{f}$ be its Lov\'{a}sz extension. For every $\x \in [0, 1]^\mathcal{N}$ and random set $D_{x} \subseteq \mathcal{N}$ obeying $\Pr[u \in D_x] = x_u$ for every $u \in \mathcal{N}$ (i.e., the marginals of $D_x$ agree with $x$), $\hat{f}(x) \leq \mathbb{E}[f(D_x)]$.
\end{lemma}

\begin{restatable}{lemma}{SiUnionA}
\label{lem:S_i_union_A}
For every $i \in [k]$ and $A \subseteq \mathcal{N}$, it holds that
\begin{equation*}
\begin{split}
\E{\f{S_i \cup A}} \geq \term{1 - \frac{1}{k}}^{\beta_i} \cdot \f{A} - \term{\term{1 - \frac{1}{k}}^{\beta_i} - \term{1 - \frac{1}{k}}^{i-1}}\cdot\f{A \cup Z}.
\end{split}    
\end{equation*}
where $\beta_i = \max\br{0, i - \ceil{t_s \cdot k} - 1}$.
\end{restatable}

\begin{proof}
Fix $i \in [k]$ and observe that each element from the set $\mathcal{N} \setminus S$ stays outside of $S_i$ with probability at least $1 - \frac{1}{k}$. This means each element belongs to $S_i$ with probability at max $1 - \term{1 - \frac{1}{k}}^{i-1}$.

To that end, to obtain the lower bounds in Lemma~\ref{lem:S_i_union_A}, we use the Lov\'{a}sz extension. Such a step requires the need for a vector in $[0,1]^{\mathcal{N}}$ representing $S_i$. We thus use the marginal vector of $S_i$, namely, $\x{S_i}$ which holds the probability of each item present in $S_i$. This in turn ensures that $\norm{\x{S_i}}_\infty \leq \term{1 - \term{1 - \frac{1}{k}}^{i-1}}$.

For every $A \subseteq \mathcal{N}$, let $\mathbf{1}_A$ denote a vector of size $\abs{\mathcal{N}}$ containing $1$s at the entries that correspond to elements present in $A$. By definition of $\T{\cdot}$, for every $\lambda \in \left(1 - \term{1 - \frac{1}{k}}^{i-1}, 1\right]$, $\T{\x{S_i} \vee \mathbf{1}_A} = A$ where $\vee$ denotes the coordinate-wise maximum operator.

Hence, 
\begin{equation*}
\begin{split}
\E{\f{S_i \cup A}} &\geq \lovasz{\x{S_i} \vee \mathbf{1}_A} = \int\limits_{0}^1 \f{\T{\x{S_i} \vee \mathbf{1}_A}} d\lambda\\
&\geq \int\limits_{\max\br{0, 1 - \term{1 - \frac{1}{k}}^{\beta_i}}}^{1 - \term{1 - \frac{1}{k}}^{i-1}} \f{\T{\x{S_i} \vee \mathbf{1}_A}} d\lambda + \int\limits_{1 - \term{1 - \frac{1}{k}}^{i-1}}^1 \f{\T{\x{S_i} \vee \mathbf{1}_A}} d\lambda \\
&= \int\limits_{\max\br{0, 1 - \term{1 - \frac{1}{k}}^{\beta_i}}}^{1 - \term{1 - \frac{1}{k}}^{i-1}} \f{\T{\x{S_i} \vee \mathbf{1}_A}} d\lambda  + \term{1 - \frac{1}{k}}^{i-1}\f{A} \\
&\geq  \term{\term{1 - \frac{1}{k}}^{\beta^i} - \term{1 - \frac{1}{k}}^{i-1}}\term{\f{A} - \f{A \cup Z}} + \term{1 - \frac{1}{k}}^{i-1}\f{A}\enspace,
\end{split}
\end{equation*}
where the last inequality follows for some $B_\lambda \subset \mathcal{N} \setminus Z$ such that $\T{\x{S_i} \vee \mathbf{1}_A} = B_\lambda \cup A$ and observing that 
\begin{equation*}
\begin{split}
\f{\T{\x{S_i} \vee \mathbf{1}_A}} = \f{B_\lambda \cup A} \geq  \f{A} + \f{A \cup B_\lambda \cap Z} - \f{A \cup Z} \geq  \f{A} - \f{A \cup Z}\enspace,
\end{split}
\end{equation*}
where the first inequality above follows from submodularity of $\f{\cdot}$, and the second inequality holds by non-negativity of $\f{\cdot}$.
\end{proof}

With the above result, we can now bound the expected value of $\f{S_i}$ for every $i \in [k]$ as follows.

\begin{restatable}{lemma}{randomGreedyExpectationBound}
\label{lem:expectation_bound}
Let $k$ a positive integer and $\alpha = 1 - \frac{1}{k}$. Then for every positive integer $i \leq \ceil{t_s \cdot k}$
\begin{equation*}
\E{\f{S_i}} \geq \term{1 - \alpha^i}\f{\OPT \setminus Z} - \term{1 - \alpha^{i} - \frac{i}{k}\alpha^{i-1}} \f{\OPT \cup Z} + \alpha^i \f{S_0} \enspace,
\end{equation*}
and for every positive integer $i \in \sterm{\ceil{t_s \cdot k} + 1, k}$ it holds that
\begin{equation*}
\E{\f{S_i}} \geq \frac{i - \ceil{t_s \cdot k}}{k} \alpha^{i - \ceil{t_s \cdot k} - 1}\f{\OPT} - \frac{i -\ceil{t_s \cdot k}}{k} \term{\alpha^{i - \ceil{t_s \cdot k} - 1} - \alpha^{i-1}}\f{\OPT \cup Z}\enspace.
\end{equation*}
\end{restatable}

\begin{proof}
Fix $i \in [k]$, and let $E_{i}$ be the event of fixing all the random decisions in Algorithm~\ref{alg:main_random_greedy} up to iteration $i-1$ (including). Let $\mathcal{E}_i$ be the set of all possible events $E_i$. Following the analysis of~\cite{buchbinder2014submodular}, we observe that 
\begin{equation*}
\begin{split}
\E{\f{u_i \mid S_{i-1}}} = k^{-1} \sum\limits_{u \in M_i}\term{\f{S_{i-1} \cup u} - \f{S_{i-1}}} &\geq k^{-1} \sum\limits_{u \in A} \term{\f{S_{i-1} \cup u} - \f{S_{i-1}}}\\
&\geq k^{-1}\term{\f{S_{i-1} \cup A} - \f{S_{i-1}}}\enspace,
\end{split}
\end{equation*}
where the last inequality holds by Algorithm~\ref{alg:main_random_greedy}, and $A = \OPT \setminus Z$ for any integer $i \leq \ceil{t_s \cdot k}$, while $A = \OPT$ otherwise. Since $\f{S_{i-1}}$ is independent of $M_i$, it holds that $\E{\f{u_i \mid S_{i-1}}} = \E{\f{S_i} - \f{S_{i-1}}} = \E{\f{S_i}} - \f{S_{i-1}}$. Hence, we obtain that $\E{\f{S_i}} \geq k^{-1}\f{S_i \cup A} + \alpha \f{S_{i-1}}$.

Unfixing $E_i$, $i$ and taking an expectation over all possible such events yields that for every $i \in [k]$, 
\[
\E{\f{u_i \mid S_{i-1}}} \geq k^{-1}\term{\E{\f{S_{i-1} \cup A}} - \E{\f{S_{i-1}}}}\enspace.
\]

\paragraph{Handling the case of $i \leq \ceil{t_s \cdot k}$.} Plugging $A:= \OPT \setminus Z$ into Lemma~\ref{lem:S_i_union_A}, yields that
\begin{equation}
\label{eq:expectation_S_i_recurrence_relation}
\begin{split}
\E{\f{S_i}} \geq \frac{1}{k}\f{\OPT \setminus Z} - \frac{1}{k}\term{1 - \alpha^{i-1}}\f{\OPT \cup Z} +\alpha \E{\f{S_{i-1}}}\enspace.
\end{split}
\end{equation}

We use induction over every positive integer $i \leq \ceil{t_s \cdot k}$. The lemma holds trivially by plugging $i := 1$ into~\eqref{eq:expectation_S_i_recurrence_relation}.  Assume that for every $ 1\leq j < i$, it holds that 
\begin{equation*}
\begin{split}
\E{\f{S_j}} &\geq \term{1 - \alpha^{j}} \f{\OPT \setminus Z} - \term{1 - \alpha^{j} - jk^{-1}\alpha^{j-1}} \f{\OPT \cup Z} + \alpha^{j} \f{S_0}\enspace.
\end{split}
\end{equation*}
Let us prove it for $i > 1$.
\begin{equation*}
\begin{split}
\E{\f{S_i}} &\geq \frac{1}{k}\f{\OPT \setminus Z} - \frac{1}{k}\term{1 - \alpha^{i-2}}\f{\OPT \cup Z} + \alpha\E{\f{S_{i-1}}}\\
&\geq  \frac{1}{k}\f{\OPT \setminus Z} - \frac{1}{k}\term{1 - \alpha^{i-2}}\f{\OPT \cup Z} + \alpha \left( \term{1 - \alpha^{i-1}} \f{\OPT \setminus Z} \right. \\
&\left.\quad - \term{1 - \alpha^{i-1} - (i-1)k^{-1}\alpha^{i-1}} \f{\OPT \cup Z} +
\alpha^{i-1} \f{S_0} \right) \\
&= \term{1 - \alpha^i}\f{\OPT \setminus Z} - \term{\alpha - \alpha^{i} - \frac{i-1}{k}\alpha^{i-1} + \frac{1}{k} - \frac{\alpha^{i-1}}{k}} \f{\OPT \cup Z} \\
&\quad\quad + \alpha^{i}\f{S_0}\\
&= \term{1 - \alpha^{i}}\f{\OPT \setminus Z} - \term{1 - \alpha^i - \frac{i}{k}\alpha^{i-1}}\f{\OPT \cup Z} + \alpha^{i} \f{S_0}\enspace,
\end{split}    
\end{equation*}
where the first and second equalities hold by definition of $\alpha$.

\paragraph{Handling the case of $i > \ceil{t_s \cdot k}$.}
Similarly, plugging $A:=\OPT$ into Lemma~\ref{lem:S_i_union_A} yields
\[
\E{\f{S_i}} \geq \frac{\alpha^{i - \ceil{t_s \cdot k} - 1}}{k}\f{\OPT} - \frac{1}{k}\term{\alpha^{i - \ceil{t_s \cdot k} - 1} - \alpha^{i-1}} \f{\OPT \cup Z} + \alpha \E{\f{S_{i-1}}}\enspace.
\]

Also here, we will prove Lemma~\ref{lem:expectation_bound} for this case using induction over $i$. For the case of $i = \ceil{t_s \cdot k} + 1$, it holds that 
\[
\E{\f{S_i}} \geq \frac{1}{k}\f{\OPT} - \frac{1}{k}\term{1 - \alpha^{\ceil{t_s \cdot k}}} \f{\OPT \cup Z} + \alpha \E{\f{S_{\ceil{t_s \cdot k}}}}\enspace.
\]

Assume that for every $ \ceil{t_s \cdot k} + 1 \leq j < i$, it holds that 
\begin{equation*}
\begin{split}
\E{\f{S_j}} &\geq \frac{j - \ceil{t_s \cdot k}}{k} \alpha^{j - \ceil{t_s \cdot k} - 1}\f{\OPT} - \frac{j - \ceil{t_s \cdot k}}{k}\term{\alpha^{j - \ceil{t_s \cdot k} - 1} - \alpha^{i-1}} \f{\OPT \cup Z} \\
&\quad\quad + \alpha^{j - \ceil{t_s \cdot k}} \E{\f{S_{\ceil{t_s \cdot k}}}}\enspace.
\end{split}
\end{equation*}
Let us prove it for $i > \ceil{t_s \cdot k} + 1$.
\begin{align*}
\E{\f{S_i}} &\geq \frac{\alpha^{i - \ceil{t_s \cdot k} - 1}}{k}\f{\OPT} - \frac{1}{k}\term{\alpha^{i - \ceil{t_s \cdot k} - 1} - \alpha^{i-1}} \f{\OPT \cup Z} + \alpha \f{S_{i-1}}\\
&\geq  \frac{i - \ceil{t_s \cdot k}}{k} \alpha^{i - \ceil{t_s \cdot k} - 1}\f{\OPT} \\
&\quad - \frac{1}{k}\term{\alpha^{i - \ceil{t_s \cdot k} - 1} - \alpha^{i-1} + \term{i - 1 - \ceil{t_s \cdot k}}\term{\alpha^{i - \ceil{t_s \cdot k} -1} - \alpha^{i - 1}}}\f{\OPT \cup Z} \\
&\quad + \alpha^{i - \ceil{t_s \cdot k}} \cdot \f{S_{\ceil{t_s \cdot k}}} \\
&= \frac{i - \ceil{t_s \cdot k}}{k} \alpha^{i - \ceil{t_s \cdot k} - 1}\f{\OPT} - \frac{i -\ceil{t_s \cdot k}}{k} \term{\alpha^{i - \ceil{t_s \cdot k} - 1} - \alpha^{i-1}}\f{\OPT \cup Z} \\
&\quad\quad + \alpha^{i - \ceil{t_s \cdot k}} \cdot \f{S_{\ceil{t_s \cdot k}}}. \qedhere
\end{align*}    
\end{proof}

We can obtain a bound on $\E{\f{S_k}}$ independent of $i$ with the above result.
\begin{restatable}{theorem}{randomGreedy}
\label{thm:random_greedy_guarantee}
There exists an algorithm that given a positive integer $k$, a value $t_s \in [0, 1]$, a non-negative submodular function $f\colon 2^{\mathcal{N}} \to \mathbb{R}$, and a set $Z \subseteq \mathcal{N}$ obeying the inequalities given in Theorem~\ref{thm:localSearchGuarantee}, outputs a solution $S_k$, obeying 
\begin{equation*}
\begin{split}
\E{\f{S_k}} &\geq \term{\frac{k - \ceil{t_s \cdot k}}{k} \alpha^{k - \ceil{t_s \cdot k} - 1} + \alpha^{k - \ceil{t_s \cdot k}} - \alpha^{k}}\f{\OPT} + \\
&\quad + \term{\alpha^k + \alpha^{k-1} - \frac{2k - \ceil{t_s \cdot k} - 1}{k} \alpha^{k - \floor{t_s \cdot k}}}\f{\OPT \cup Z} \\
&\quad + \term{\alpha^k - \alpha^{k - \ceil{t_s \cdot k}} } \f{\OPT \cap Z}.
\end{split}
\end{equation*}
Furthermore, this algorithm requires only $O(nk)$ queries to the objective function.
\end{restatable}

\begin{proof}
First, note that by submodularity and non-negativity of $\f{\cdot}$,
\[
\f{\OPT \setminus Z} \geq \f{\emptyset} + \f{\OPT} - \f{\OPT \cap Z} \geq \f{\OPT} - \f{\OPT \cap Z}.
\]
The lower bound in Theorem~\ref{thm:random_greedy_guarantee} follows by combining the above inequality with the following two inequalities: the inequality arising by plugging $i := \ceil{t_s \cdot k}$ into the first case of Lemma~\ref{lem:expectation_bound}, and the inequality arising by plugging $i := k$ into the second case of Lemma~\ref{lem:expectation_bound}. Note that we removed all terms involving $\f{S_0}$ as they are non-negative.

To conclude the proof of Theorem~\ref{thm:random_greedy_guarantee}, observe that Algorithm~\ref{alg:main_random_greedy} at each iteraion $i \in [k]$, goes over $O\term{n}$ items, where the marginal gain is computed separately for each of the items with respect to the set $S_{i-1}$. This process is repeated $O(k)$ times, resulting in a total of $O(nk)$ queries to the objective function.
\end{proof}

\subsection{0.385-Approximation Guarantee}
In what follows, we present the Algorithm that is used to prove Theorem~\ref{thm:main_simple}. The algorithm returns the better among the two sets produced in the last two steps (i.e., the output sets of Algorithm~\ref{alg:local_search}, and Algorithm~\ref{alg:main_random_greedy}).

\begin{algorithm}[H]
    \SetKwInOut{Input}{input}
    \SetKwInOut{Output}{output}

    \Input{A positive integer $k \geq 1$, a submodular function $f$, $t_s \in [0,1]$ a timestamp threshold, and probability $p \in [0,1]$}
    \Output{A set $S \subseteq \mathcal{N}$}

    $S_{1} \gets \localsearch\term{k, f}$\\
    $S_2 \gets $ the output of the execution of either Algorithm~\ref{alg:main_random_greedy} or Algorithm~\ref{alg:main_lazier_than_lazy_greedy}\\
    
    \Return $\max\{f(Z),f(A)\}$
    \caption{Warmup Algorithm}
    \label{alg:main_algorithm_simple}
\end{algorithm}
Next, we present the approximation guarantee of this algorithm.

\begin{theorem}[Approximation guarantee] \label{thm:main_simple}
Given an integer $k \geq 1$ and a non-negative submodular function $f\colon 2^{\mathcal{N}} \to \mathbb{R}$, there exists an $0.385$-approximation algorithm for the problem of finding a set $S \subseteq \mathcal{N}$ of size at most $k$ maximizing $f$. This algorithm uses $O_\eps\term{n \cdot k^2}$ queries to the objective function.
\end{theorem}
\begin{proof}
    Aside from the number of queries required by this algorithm, which follows directly from Theorem~\ref{thm:localSearchGuarantee}, the approximation guarantee follows directly from the proof Theorem~\ref{thm:main}, which can be found in Appendix~\ref{sec:proofs_main}. 
\end{proof}
\section{Proof of technical theorems and lemmata}~\label{sec:proofs_main}
In this section, we will prove Theorem~\ref{thm:main_simple}. As explained in Section~\ref{sec:main_result}, the proof of Theorem~\ref{thm:main_simple} is based on two Theorems, namely Theorem~\ref{thm:fast_local_search}, and Theorem~\ref{thm:lazy_greedy_guarantee}. We first start by proving these Theorems, and immediately after that we will prove Theorem~\ref{thm:main_simple}. 
\subsection{Proof of Theorem~\ref{thm:fast_local_search}}
In this section, we prove Theorem~\ref{thm:fast_local_search}. We begin with the following lemma. We assume without loss of generality that $\OPT$ is of size $k$ (otherwise, we add to $\OPT$ dummy elements).

\begin{lemma}
\label{lem:fast_local_search_per_outer_iteration}
If the set $S_0$ provides $c$-approximation, then each iteration of the loop starting on Line~\ref{line:attempts_loop} in Algorithm~\ref{alg:faster_local_search} that starts returns a set with probability at least $k / (c\eps(1 - \nicefrac{1}{e})L)$. Moreover, when this happens, the output set $S$ returned obeys
\[\f{S} \geq \frac{\f{S \cap \OPT} + \f{S \cup \OPT}}{2 + \eps} \quad \text{and} \quad  \f{S} \geq \frac{\f{S \cap \OPT}}{1 + \eps}\enspace.\]
\end{lemma}

\begin{proof}
For every two integers $0 \leq i < L$ and integer $1 \leq j \leq \ceil{\log{\frac{1}{\eps}}}$, we denote by $A_i^j$ the event that the set $S_{i, j}$ obeys the condition on Line~\ref{line:condition} of Algorithm~\ref{alg:faster_local_search}, i.e., the event that for every integer $0 \leq t \leq k$ it holds that
\[
\max\limits_{S \subseteq \mathcal{N} \setminus S^j_{i}, \abs{S} = t} \sum\limits_{u \in S} \f{u \mid S^j_{i}} \leq \min_{S \subseteq S^j_{i}, \abs{S} = t} \sum\limits_{v \in S} \f{v \mid S^j_{i} - v} + \eps\f{S^j_{i}}\enspace.
\]
To better understand the implication of event $A^j_i$, assume that such an event occurs, and observe that for $t := \abs{\OPT \setminus S^j_i}$, it holds that
\begin{equation}
\label{eq:fast_local_search_implication1}
\begin{split}
\f{S^j_i} - \f{S^j_i \cap \OPT} + \eps \f{S^j_i} &=\sum\limits_{\ell=1}^{\abs{K^j_{L}}} \f{u_\ell\mid J^j_L\cup\{u_1,\ldots,u_{\ell-1}\}} + \eps \f{S^j_i}  \\
&\geq \sum\limits_{\ell=1}^{\abs{K^j_{L}}}\f{u_\ell\mid J^j_{L}\cup K^j_{L} - u_\ell} + \eps \f{S^j_i}\\
&= \sum\limits_{v \in S^j_i \setminus \OPT} \f{v \mid S_i^j - v} + \eps \f{S^j_i} \\
&\geq \min_{S \subseteq S^j_i, \abs{S} = t} \sum\limits_{v \in S} \f{v \mid S - v} + \eps\f{S^j_i} \\
&\geq \max\limits_{S \subseteq \mathcal{N} \setminus S^j_{i}, \abs{S} = t} \sum\limits_{u \in S} \f{u \mid S^j_{i}} \\
&\geq \sum\limits_{u \in \OPT \setminus S^j_{i}} \f{u \mid S^j_{i}}\\
&\geq \f{S^j_i \cup \OPT} - \f{S^j_i}\enspace ,
\end{split}
\end{equation}
where the first equality holds by setting $J^j_{i} := S^j_{i} \cap \OPT$ and $K^j_{i} := \br{u_\ell}_{\ell=1}^{\abs{K_{i}}} = S^j_{i} \setminus \OPT$, the first and last inequalities follow from submodularity of $\f{\cdot}$, the second inequality holds since the fact that $S^j_i$ and $\OPT$ are both of size $k$ implies that $t = |\OPT \setminus S^j_i| = |S^j_i \setminus \OPT|$, and the third inequality holds under the assumption that event $A^j_i$ occurs. Rearranging the last inequality, we get 
\begin{equation*}
\f{S} \geq \frac{\f{S \cup \OPT} + \f{S \cap \OPT}}{2 + \eps}\enspace .
\end{equation*}

In addition, the existence of the dummy elements implies that $\max\limits_{S \subseteq \mathcal{N} \setminus S^j_{i}, \abs{S} = t} \sum\limits_{u \in S} \f{u \mid S^j_{i}} \geq 0$, and plugging this inequality into the previous one yields that the event $A^j_i$ also implies
\begin{equation*}
\f{S^j_i} - \f{S^j_i \cap \OPT} + \eps\f{S^j_i} \geq \min_{S \subseteq S^j_i, \abs{S} = t} \sum\limits_{v \in S} \f{v \mid S^j_i - v} + \eps\f{S^j_i} \geq 0\enspace,
\end{equation*}
and rearranging this inequality gives $\f{S^j_i} \geq \frac{\f{S^j_i \cap \OPT}}{1 + \eps}$.

The above shows that to prove the lemma it suffices to show that the probability that $A^j_{i^*}$ holds is at least $k / (c\eps(1 - \nicefrac{1}{e})L)$. Towards this goal, let us study the implications of the complementary event $\bar{A^j_i}$. Specifically, we would like to lower bound $\E{\f{S^j_{i + 1}} - \f{S^j_i} \mid \bar{A^j_i}}$.

Fix a particular set $S^j_i$ that causes the event $\bar{A^j_i}$ to occur (notice that the occurrence of this event depends only on the set $S^j_i$). Then, there must exist sets $T_+ \subset \mathcal{N} \setminus S^j_i$ and $T_- \subseteq S^j_i$ of size $t \leq k$ such that
\[
\sum\limits_{u \in T_+} \f{u \mid S^j_{i}} > \sum\limits_{v \in T_-} \f{v \mid S^j_{i} - v} + \eps\f{S^j_{i}}\enspace.
\]

We can now define the event $B_i^j$ as the event that $Z^j_{i + 1} \cap T_+ \neq \varnothing$ (notice that the event $B_i^j$ is defined only for this particular set $S^i_j$).
The probability of the event $B^j_i$ is
\[
\begin{split}
\pr{B^j_i \mid S^j_i} \geq 1 - \term{\frac{n - \nicefrac{n}{k}}{n}}^{\abs{T_+}} \geq 1 - e^{-\frac{\abs{T_+}}{k}} \geq \term{1 - \nicefrac{1}{e}} \frac{\abs{T_+}}{k}\enspace,
\end{split}
\]
where the second inequality holds since $\term{1 - \frac{1}{k}}^x \leq e^{-\frac{x}{k}}$ for any $x \geq 0$, and the second inequality holds for any $x \in \sterm{0,1}$ by the concavity of $1 - e^{-x}$. By the law of total expectation and the fact that $\f{S^j_{i + 1}}$ is always at least $\f{S^j_i}$, we now get
\begin{equation*}
\begin{split}
\E{\f{S^j_{i + 1}} - \f{S^j_i} \mid S^j_i} &\geq \pr{B^j_i \mid S^j_i} \E{\f{S^j_{i + 1}} - \f{S^j_i} \mid S^j_i, B^j_i}\\
&\geq  \term{1 - \nicefrac{1}{e}} \frac{\abs{T_+}}{k} \cdot \E{\f{S^j_{i + 1}} - \f{S^j_i} \mid S^j_i, B^j_i}\\ 
&\geq \frac{1 - \nicefrac{1}{e}}{k} \eps c \f{\OPT}\enspace ,
\end{split}
\end{equation*}
where the last inequality holds since
\begin{equation*}
\begin{split}
&\E{\f{S^j_{i + 1}} - \f{S^j_{i}} \mid S^j_i, B^j_i} = \E{\f{S^j_{i} - v_{i + 1}^j + u_{i + 1}^j} - \f{S^j_{i}} \mid S^j_i, B^j_i} \\
&\quad\quad\quad\quad= \E{\f{S^j_{i} - v_{i + 1}^j + u_{i + 1}^j} - \f{S^j_{i} - v_{i + 1}^j}+ \f{S^j_{i} - v_{i + 1}^j} - \f{S^j_{i}} \mid S^j_i, B^j_i}\\ 
&\quad\quad\quad\quad \geq \E{\f{S^j_{i} + u_{i + 1}^j} - \f{S^j_{i}} + \f{S^j_{i} - v_{i + 1}^j} - \f{S^j_{i}} \mid S^j_i, B_i^j} \\
&\quad\quad\quad\quad \geq \E{\f{S^j_{i} + u'} - \f{S^j_{i}} \mid S^j_i, B^j_i} + \E{\f{S^j_{i} - v'} - \f{S^j_{i}} \mid S^j_i, B^j_i}\\
&\quad\quad\quad\quad = \frac{\sum\limits_{u \in T_+} \f{u \mid S^j_{i}}}{\abs{T_+}} -  \frac{\sum\limits_{v \in T_-}\f{v \mid S_{i}^j - v}}{\abs{T_-}}\\
&\quad\quad\quad\quad \geq \frac{\eps\f{S^j_{i}}}{\abs{T_+}}\\
&\quad\quad\quad\quad \geq  \frac{c\eps\f{\OPT}}{\abs{T_+}}\enspace,
\end{split}
\end{equation*}
where the first inequality holds by submodularity of $\f{\cdot}$, the second inequality holds by the way Algorithm~\ref{alg:faster_local_search} chooses $u^j_i$ and $v^j_i$ if we let $u'$ be a uniformly random element of $T_+ \cap Z^j_{i + 1}$ and $v'$ be a uniformly random element of $T_-$, the penultimate inequality holds by our assumption that $S^j_i$ implies the event $\bar{A^j_i}$, and finally, the last inequality holds since it is guaranteed that $\f{S^j_{i}} \geq \f{S^j_0} = \f{S_0} \geq c \cdot f(\OPT)$. 

Since the above bound on the expectation holds conditioned on every set $S_i^j$ that implies the event $\bar{A^j_i}$, it holds (by the law of total expectation) also conditioned on the event $\bar{A^j_i}$. Adding this lower bound for all $i$ values, and using the non-negativity of $f$, we get
\begin{equation*}
\begin{split}
\E{\f{S^j_L}} &\geq \sum\limits_{\ell = 0}^{L - 1} \E{\f{S^j_{\ell + 1}} - \f{S^j_{\ell}}} \geq \sum\limits_{\ell = 0}^{L - 1} \pr{\bar{A^j_\ell}} \E{\f{S^j_{\ell + 1}} - \f{S^j_{\ell}} \mid \bar{A^j_\ell}}\\
&\geq \frac{1 - \nicefrac{1}{e}}{k}c\eps\f{\OPT} \sum\limits_{\ell = 1}^L \pr{\bar{A^j_\ell}}\enspace .
\end{split}
\end{equation*}
Combining the last inequality with the fact that $\f{S^j_L}$ is deterministically at most $\f{\OPT}$, it must hold that
\begin{equation*}
\begin{split}
1 \geq \frac{1 - \nicefrac{1}{e}}{k}c\eps \sum\limits_{\ell = 1}^L \pr{\bar{A^j_\ell}} \quad \Rightarrow \quad  \sum\limits_{\ell = 1}^L \pr{\bar{A^j_\ell}} \leq \frac{k}{c\eps \term{1 - \nicefrac{1}{e}}}\enspace .
\end{split}
\end{equation*}
Hence, the probability that the event $A_{i^*}^j$ does not hold for a uniformly random $i^* \in [L]$ is    
\[
\frac{\sum_{\ell = 1}^L \pr{\bar{A_\ell^j}}}{L} \leq \frac{k}{c\eps\term{1 - \nicefrac{1}{e}} L}\enspace.
\qedhere
\]
\end{proof}

\thmfastlocalsearch*

\begin{proof}
As mentioned above, we initialize the set $S_0$ using $\lceil \log_2 \eps^{-1} \rceil$ executions of the Sample Greedy algorithm of~\cite{buchbinder2017comparing}. This algorithms returns a set that provides $(1/e - \eps')$-approximation in expectation for a fixed constant $\eps'$ using $O_{\eps'}(n)$ queries. Setting $\eps' = 1/e - 1/4$, we get by Markov's inequality that the output set of Sample Greedy provides at least $1/8$-approximation with probability at least $1/2$. Thus, by selecting the best set produced by $\lceil \log_2 \eps^{-1} \rceil$ executions of Sample Greedy, we get a set $S_0$ that provides at least $c:=1/8$-approximation with probability at least $1 - \eps$. Furthermore, this requires only $O_\eps(n)$ queries.

Assume that the set $S_0$ indeed provides $c$-approximation, and let us now set $L := \left\lceil\frac{2k}{c\eps\term{1 - \nicefrac{1}{e}}}\right\rceil$ in Algorithm~\ref{alg:faster_local_search}. Then,  Lemma~\ref{lem:fast_local_search_per_outer_iteration} guarantees that every iteration of the outer loop of the algorithm returns a set (obeying the requirement of the theorem) with probability at least $1/2$. Hence, by repeating this loop $\lceil \log_2 \eps^{-1}\rceil$ times, we are guaranteed that Algorithm~\ref{alg:faster_local_search} outputs a set with probability at least $1 - \eps$ (condition on the set $S_0$ providing $c$-approximation).

Combining the results of the last two paragraphs, we get that Algorithm~\ref{alg:faster_local_search} succeeds in outputing a set $S$ obeying the inequalities of the theorem with probability at least $(1 - \eps)^2 \geq 1 - 2\eps$. This probability can be improved to $1 - \eps$, as required by the theorem, by halving the value of $\eps$.
To complete the proof of the theorem, it only remains to observe that for the above choices of $c$ and $L$, Algorithm~\ref{alg:faster_local_search} can be implemented using only $O_\eps(n + k^2)$ queries to the objective function.
\end{proof}

\subsection{Proof of Lemma~\ref{lem:expectation_bound_lazy_greedy}}

One can verify that Algorithm~\ref{alg:main_lazier_than_lazy_greedy} guarantees that, for any $i \in [k]$, the probability of any element $u \in \mathcal{N}$ to belong to $S_i$ is at most $1 - \term{1 - \frac{1}{k}}^{i-1}$ (see the analysis of Sample Greedy in~\cite{buchbinder2017comparing} for details), which is similar to the guarantee on the probabilities in Algorithm~\ref{alg:main_random_greedy}. This means that the proof of Lemma~\ref{lem:S_i_union_A} applies also to Algorithm~\ref{alg:main_lazier_than_lazy_greedy}, and thus, we can use this lemma to prove the following result.

\begin{restatable}{lemma}{expectationBoundLazyGreedy}
\label{lem:expectation_bound_lazy_greedy}
Let $\eps \in (0,1)$, $k$ a positive integer and $\alpha = 1 - \frac{1}{k}$. Then for every positive integer $i \leq \ceil{t_s \cdot k}$
\begin{equation*}
\E{\f{S_i}} \mspace{-2mu}\geq\mspace{-2mu} \term{1 - \alpha^i\mspace{-2mu}}\f{\OPT \setminus Z} \mspace{-2mu}-\mspace{-2mu} \term{1 - \alpha^{i} - \frac{i\alpha^{i-1}}{k}} \mspace{-2mu}\f{\OPT \cup Z} \mspace{-2mu}+\mspace{-2mu} \alpha^i \f{S_0} 
\mspace{-2mu}-\mspace{-2mu} \frac{2\eps \sum\limits_{l=1}^i\alpha^{l-1}}{k}\mspace{-2mu}\enspace,
\end{equation*}
and for every positive integer $i \in \sterm{\ceil{t_s \cdot k} + 1, k}$, 
\begin{equation*}
\begin{split}
\E{\f{S_i}} &\geq \frac{i - \ceil{t_s \cdot k}}{k} \alpha^{i - \ceil{t_s \cdot k} - 1}\f{\OPT} - \frac{i -\ceil{t_s \cdot k}}{k} \term{\alpha^{i - \ceil{t_s \cdot k} - 1} - \alpha^{i-1}}\f{\OPT \cup Z} \\
&\quad + \alpha^{i - \ceil{t_s \cdot k}} \f{S_{\ceil{t_s \cdot k}}}  - \frac{2\eps\sum\limits_{l=\ceil{t_s k} + 1}^{i} \alpha^{l-\ceil{t_s k} - 1}}{k}\f{\OPT} \enspace.
\end{split}
\end{equation*}

\end{restatable}
\begin{proof}
Fix $i \in [k]$, and let $E_{i}$ be the event of fixing all the random decisions in the algorithm up to iteration $i-1$ (including). Let $\mathcal{E}_i$ be the set of all possible events $E_i$. Following the analysis of the \emph{Sample Greedy} algorithm~\cite{buchbinder2017comparing}, specifically Lemma 4.3, we obtain that 
\begin{equation*}
\begin{split}
\E{\f{u_i \mid S_{i-1}}} \geq \frac{1-\eps}{k} \term{\f{A \cup S_{i-1}} - \f{S_{i-1}}}\enspace,
\end{split}
\end{equation*}
where $A := \OPT \setminus Z$ when $i \leq \ceil{t_s \cdot k}$, and $A := \OPT$ otherwise. Furthermore, note that since $\f{S_{i-1}}$ is independent of $M_i$, then it holds that $\E{\f{u_i \mid S_{i-1}}} = \E{\f{S_i} - \f{S_{i-1}}} = \E{\f{S_i}} - \f{S_{i-1}}$. This gives
\begin{equation*}
\begin{split}
\E{\f{S_i}} &\geq \term{1 - \alpha - \frac{\eps}{k}} \f{S_{i-1} \cup A} + \term{\alpha 
 + \frac{\eps}{k}}\f{S_{i-1}} \\
 &\geq \term{1 - \alpha}\f{S_{i-1} \cup A} + \alpha{\f{S_{i-1}}} - \frac{2\eps}{k}\f{\OPT}\enspace,
\end{split}
\end{equation*}
where $\alpha := 1 - \frac{1}{k}$, and the second inequality holds by submodularity of $\f{\cdot}$.

Unfixing event $E_i$ and $i$, yields that for every $i \in [k]$
\begin{equation}
\label{eq:lazier_recurrence}
\begin{split}
\E{\f{S_i}} \geq \term{1 - \alpha}\E{\f{S_{i-1} \cup A}} + \alpha\E{\f{S_{i-1}}} - \frac{2\eps}{k}\f{\OPT}\enspace.
\end{split}
\end{equation}

To obtain the guarantees of Lemma~\ref{lem:expectation_bound_lazy_greedy}, we inspect two main cases.

\paragraph{Handling the case of $i \leq \ceil{t_s \cdot k}$.} Combining the result of plugging $A := \OPT \setminus Z$  into Lemma~\ref{lem:S_i_union_A} with~\eqref{eq:lazier_recurrence} yields that 
\begin{equation}
\label{eq:expectation_S_i_recurrence_relation_lazy}
\begin{split}
\E{\f{S_i}} &\geq \term{1 - \alpha} \f{\OPT \setminus Z} - \frac{1}{k}\term{1 - \alpha^{i-1}}\f{\OPT \cup Z} \\
&\quad+\alpha \E{\f{S_{i-1}}} -\frac{2\eps}{k}\f{\OPT}\enspace.
\end{split}
\end{equation}

We now prove Lemma~\ref{lem:expectation_bound_lazy_greedy} for the case above via induction over every positive integer $i \leq \ceil{t_s \cdot k}$. The lemma holds trivially when plugging $i := 1$ into~\eqref{eq:expectation_S_i_recurrence_relation_lazy}.  Assume that for every $ 1 \leq j < i$, it holds that 
\begin{equation*}
\begin{split}
\E{\f{S_j}} &\geq \term{1 - \alpha^{j}} \f{\OPT \setminus Z} - \term{1 - \alpha^{j} - jk^{-1}\alpha^{j-1}} \f{\OPT \cup Z} \\
&\quad + \alpha^{j} \f{S_0} - \frac{2\eps \sum\limits_{l=1}^j\alpha^{l-1}}{k}\f{\OPT}\enspace.
\end{split}
\end{equation*}
Let us prove it for $i > 1$.

\begin{equation*}
\begin{split}
\E{\f{S_i}} &\geq \frac{1}{k}\f{\OPT \setminus Z} - \frac{1}{k}\term{1 - \alpha^{i-2}}\f{\OPT \cup Z} + \alpha\E{\f{S_{i-1}}} - \frac{2\eps}{k}\f{\OPT}\\
&\geq  \frac{1}{k}\f{\OPT \setminus Z} - \frac{1}{k}\term{1 - \alpha^{i-2}}\f{\OPT \cup Z} + \alpha \left( \vphantom{\frac{2\eps\sum\limits_{l=1}^{i-1}\alpha^{l-1}}{k}}\term{1 - \alpha^{i-1}} \f{\OPT \setminus Z}  \right.\\
&-\left. \term{1 - \alpha^{i-1} - (i-1)k^{-1}\alpha^{i-1}} \f{\OPT \cup Z} +
\alpha^{i-1} \f{S_0} - \frac{2\eps\sum\limits_{l=1}^{i-1}\alpha^{l-1}}{k}\f{\OPT} \right) \\
&= \term{1 - \alpha^i}\f{\OPT \setminus Z} - \term{\alpha - \alpha^{i} - \frac{i-1}{k}\alpha^{i-1} + \frac{1}{k} - \frac{\alpha^{i-1}}{k}} \f{\OPT \cup Z} \\
&\quad\quad  + \alpha^{i}\f{S_0} - \frac{2\eps\sum\limits_{l=1}^i \alpha^{l-1}}{k}\f{\OPT}\\
&= \term{1 - \alpha^{i}}\f{\OPT \setminus Z} - \term{1 - \alpha^i - \frac{i}{k}\alpha^{i-1}}\f{\OPT \cup Z} + \alpha^{i} \f{S_0} \\
&\quad\quad - \frac{2\eps\sum\limits_{l=1}^i \alpha^{l-1}}{k}\f{\OPT}\enspace,
\end{split}    
\end{equation*}
where the equalities hold by definition of $\alpha$.

\paragraph{Handling the case of $i > \ceil{t_s \cdot k}$.}
Plugging $A := \OPT $  into Lemma~\ref{lem:S_i_union_A} yields that
\begin{equation*}
\begin{split}
\E{\f{S_i}} &\geq \frac{\alpha^{i - \ceil{t_s \cdot k} - 1}}{k}\f{\OPT} - \frac{1}{k}\term{\alpha^{i - \ceil{t_s \cdot k} - 1} - \alpha^{i-1}} \f{\OPT \cup Z} \\
&\quad + \alpha \f{S_{i-1}} - \frac{2\eps}{k}\f{\OPT}\enspace.  
\end{split}
\end{equation*}

Similarly, we prove Lemma~\ref{lem:expectation_bound_lazy_greedy} for this case using induction over $i$. For the case of $i = \ceil{t_s \cdot k} + 1$, it holds that 
\[
\E{\f{S_i}} \geq \frac{1}{k}\f{\OPT} - \frac{1}{k}\term{1 - \alpha^{\ceil{t_s \cdot k}}} \f{\OPT \cup Z} + \alpha \f{S_{\ceil{t_s \cdot k}}} - \frac{2\eps}{k}\f{\OPT}\enspace.
\]

Assume that for every $\ceil{t_s \cdot k} + 1 \leq j < i$, it holds that 
\begin{equation*}
\begin{split}
\E{\f{S_j}} &\geq \frac{j - \ceil{t_s \cdot k}}{k}\cdot \alpha^{j - \ceil{t_s \cdot k} - 1}\cdot \f{\OPT} - \frac{j - \ceil{t_s \cdot k}}{k}\cdot\term{\alpha^{j - \ceil{t_s \cdot k} - 1} - \alpha^{i-1}}\cdot \f{\OPT \cup Z} \\
&\quad\quad + \alpha^{j - \ceil{t_s \cdot k}} \cdot\f{S_{\ceil{t_s \cdot k}}} - \frac{2\eps\sum\limits_{l=\ceil{t_s k} + 1}^{j} \alpha^{l-\ceil{t_s k} - 1}}{k}\cdot\f{\OPT}\enspace.
\end{split}
\end{equation*}
Let us prove it now for $i > \ceil{t_s \cdot k} + 1$.

\begin{equation*}
\begin{split}
\E{\f{S_i}} &\geq \frac{\alpha^{i - \ceil{t_s \cdot k} - 1}}{k}\f{\OPT} - \frac{1}{k}\term{\alpha^{i - \ceil{t_s \cdot k} - 1} - \alpha^{i-1}} \f{\OPT \cup Z} + \alpha \f{S_{i-1}} \\
&\quad - \frac{2\eps}{k}\f{\OPT}\\
&\geq  \frac{i - \ceil{t_s \cdot k}}{k} \alpha^{i - \ceil{t_s \cdot k} - 1}\f{\OPT} \\
&\quad - \frac{1}{k}\term{\alpha^{i - \ceil{t_s \cdot k} - 1} - \alpha^{i-1} + \term{i - 1 - \ceil{t_s \cdot k}}\term{\alpha^{i - \ceil{t_s \cdot k} -1} - \alpha^{i - 1}}}\f{\OPT \cup Z} \\
&\quad + \alpha^{i - \ceil{t_s \cdot k}} \cdot \f{S_{\ceil{t_s \cdot k}}}  - \frac{2\eps\sum\limits_{l=\ceil{t_s k} + 1}^{i-1} \alpha^{l-\ceil{t_s k} - 1}}{k}\f{\OPT}\\
&= \frac{i - \ceil{t_s \cdot k}}{k} \alpha^{i - \ceil{t_s \cdot k} - 1}\f{\OPT} - \frac{i -\ceil{t_s \cdot k}}{k} \term{\alpha^{i - \ceil{t_s \cdot k} - 1} - \alpha^{i-1}}\f{\OPT \cup Z} \\
&\quad\quad + \alpha^{i - \ceil{t_s \cdot k}} \cdot \f{S_{\ceil{t_s \cdot k}}}  - \frac{2\eps\sum\limits_{l=\ceil{t_s k} + 1}^{i-1} \alpha^{l-\ceil{t_s k} - 1}}{k}\f{\OPT}\enspace. \qedhere
\end{split}    
\end{equation*}
\end{proof}

\subsection{Proof of Theoream~\ref{thm:lazy_greedy_guarantee}}
Following the previous subsection, we are now ready to prove Theorem~\ref{thm:lazy_greedy_guarantee}.

\lazierGreedy*

\begin{proof}[Proof of Theorem~\ref{thm:lazy_greedy_guarantee}]
First, note that by submodularity and non-negativity of $\f{\cdot}$, it holds that for every $A \subseteq \mathcal{N}$,
\begin{equation} \label{eq:seperation}
\f{\OPT \setminus A} \geq \f{\emptyset} + \f{\OPT} - \f{\OPT \cap A} \geq \f{\OPT} - \f{\OPT \cap A}\enspace.
\end{equation}

The lower bound in Theorem~\ref{thm:lazy_greedy_guarantee} follows by combining three inequalities: the inequality arising by plugging $i := \ceil{t_s \cdot k}$ into the first case of Lemma~\ref{lem:expectation_bound_lazy_greedy}, the inequality arising by plugging $i := k$ into the second case of Lemma~\ref{lem:expectation_bound_lazy_greedy}, and Inequality~\eqref{eq:seperation} for $A := Z$. Note that we have removed all terms involving $\f{S_0}$ as they are non-negative.

To conclude the proof of Theorem~\ref{thm:lazy_greedy_guarantee}, observe that Algorithm~\ref{alg:main_lazier_than_lazy_greedy} at each iteration $i \in [k]$ samples $O\term{\frac{n\ln{\eps^{-1}}}{k\eps^{2}}}$ items, where the marginal gain is computed separately for each of the sampled elements with respect to the set $S_{i-1}$. This process is repeated $O\term{k}$ times, thus, resulting in a total of $O\term{n\eps^{-2}\ln{\eps^{-1}}}$ queries of the objective function.
\end{proof}

\section{Additional experiments}
\subsection{Image summarization}

\begin{figure}[!htb]
    \centering
    \includegraphics[width=\textwidth]{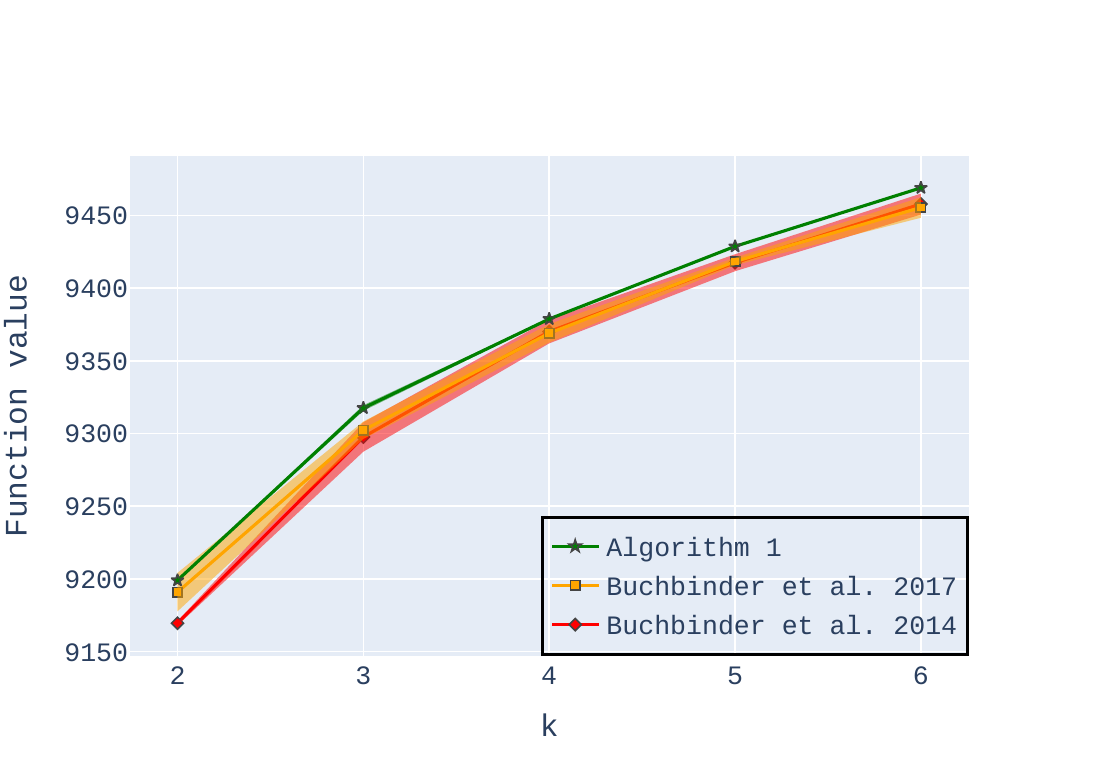}
    \caption{Results for a varying number of images $k$ with respect to the Image summarization problem on the Tiny ImageNet dataset.}
    \label{fig:enter-label}
\end{figure}

\end{document}